\documentclass[a4paper,11pt]{article}
\pdfoutput=1
\usepackage{a4wide}
\usepackage{amsmath,amsfonts,amssymb,amsthm}
\usepackage[colorlinks,citecolor=blue]{hyperref}
\usepackage{enumitem}
\usepackage{bbm}
\usepackage{natbib}
\usepackage{microtype}
\usepackage[title,titletoc,toc]{appendix}
\usepackage{mathrsfs}
\setcitestyle{square}

\newtheorem{theorem}{Theorem}[section]
\newtheorem{lemma}[theorem]{Lemma}

\theoremstyle{definition}
\newtheorem{definition}[theorem]{Definition}

\newtheorem{remark}[theorem]{Remark}

\numberwithin{equation}{section}
\numberwithin{table}{section}

\def\ZZ{\mathbb{{Z}}}
\def\RR{\mathbb{{R}}}
\def\NN{\mathbb{{N}}}

\def\PP{\mathbb{{P}}}
\def\BB{\mathbb{{B}}}
\def\DD{\mathbb{{D}}}

\begin{document}
\title{Error bounds in Sobolev norms for approximations with norm constrained ReLU neural networks}
\author{
Xianjun Li \thanks{School of Statistics and Data Science, Capital University of Economics and Business, Beijing, China. E-mail: lixianjun@cueb.edu.cn.}
\and
Yunfei Yang \thanks{School of Mathematics (Zhuhai) and Guangdong Province Key Laboratory of Computational Science, Sun Yat-sen University, Zhuhai, China. Corresponding author, E-mail: yangyunfei@mail.sysu.edu.cn.}
}
\date{}
\maketitle

\begin{abstract}
Recent studies have shown that smooth functions can be well approximated by ReLU neural networks with path norm constraint on the weights. We extend these results from uniform approximation to approximation in Sobolev norm. Specifically, we analyze how well Sobolev functions in $W^{n,p}$ can be approximated by neural networks with width $W$, depth $L$ and path norm bounded by $K$, when the approximation error is measured in the $W^{1,p}$-norm. For shallow networks with depth $L=1$, we derive the approximation error bound $\mathcal{O}(\max\{W^{-(n-1)/d}, K^{-(n-1)/(s-n)}\})$, when the smoothness index satisfies $n<s=(d+3)/2$ and the input is $d$-dimensional. For deep networks, we remove the restriction on the smoothness by showing that the approximation bound $\mathcal{O}(K^{-(n-1)/(d+d/p+1)})$ holds if the width $W$ and depth $L$ are sufficiently large.

\end{abstract}

\section{Introduction}

Neural networks have become fundamental tools in modern machine learning, with wide-ranging applications in
including image processing \citep{krizhevsky2017imagenet,lecun2015deep}, speech recognition \citep{dahl2012context,hinton2012deep},
and natural language processing \citep{young2018recent}. The approximation capacity of neural networks has been a central research theme for decades, from  universal approximation properties of shallow networks \citep{cybenko1989approximation, hornik1991approximation, pinkus1999approximation} to  approximation rates for deep networks across various function classes \citep{yarotsky2017error, yarotsky2018optimal, petersen2018optimal, yarotsky2020phase, shen2020deep,lu2021deep,montanelli2021deep}.

In practice, however, neural networks are trained with regularization techniques that impose Lipschitz constraints or bounds on the weight norms, which control generalization and robustness \citep{bartlett1998sample, bartlett2017spectrally, neyshabur2015norm, golowich2020size,cisse2017parseval}. These regularization methods often impose explicit or implicit constraints on some norms of the weights, which substantially limit the model's expressivity. Indeed, as shown in \citep{huster2019limitations}, ReLU neural networks with certain weight constraints cannot represent simple functions such as the absolute value. This limitation is particularly critical when the goal is not merely to approximate a function, but also its derivatives, which is a requirement that arises naturally in the context of solving partial differential equations (PDEs) by deep learning-based methods \citep{yu2018deep,sirignano2018dgm,kutyniok2022theoretical,lagaris1998artificial,lu2022apriori,siegel2023greedy}.
This raises the question of how norm constraints on the weights affect the approximation capacity of neural networks in Sobolev norms.

To contextualize our work, we briefly review the relevant literature. Extensive results have established approximation rates for both shallow and deep neural networks in terms of the number of parameters (weights or neurons) or network size (e.g., width or depth), with error measured in various function-space norms \citep{devore1996approximation, petrushev1998approximation, yarotsky2017error, guhring2020error, lu2021deep, shen2020deep, shen2022optimal, yang2023nearlyop, yang2023nearly, siegel2023optimal, yang2025onoptimal}. However, these results do not address the effect of explicit norm constraints on the network weights themselves.
For shallow networks, approximation results under various  norm constraints have been explored from different perspectives,
including  \citep{bach2017breaking, weinan2019priori, weinan2022barron, liu2024learning,siegel2024sharp, siegel2025optimal, mao2026approximation,li2026shallow}. In particular, \citet{yang2025optimal} derived a rate $\mathcal{O}(\max\{W^{-n/d}, K^{-2n/(d+3-2n)}\})$ for smooth functions in the $L^\infty$-norm for shallow ReLU networks with width $W$ and a path norm constraint $K$.
For deep networks, approximation bounds under explicit norm constraints have also been established in \citep{jiao2023approximation, maiale2025approximation, huang2026learning}, also primarily in the $L^\infty$-norm.
A separate line of work has considered bounds on individual weight magnitudes \citep{petersen2018optimal,boelcskei2019optimal,schmidthieber2021kolmogorov,elbrachter2021deep, guhring2021approximation, de2021approximation, belomestny2023simultaneous}.  To the best of our knowledge, the case of approximation in Sobolev norms for norm constrained networks remains open.

In this paper, we address this gap by deriving approximation rates for both shallow and deep norm constrained ReLU neural networks. Specifically, for functions in the unit ball of the Sobolev space  $W^{n,p}$ with smoothness index $n\in\mathbb{N}_{\ge 2}$, we obtain rates on the unit ball $\mathbb{B}^d$ for shallow norm constrained networks and on the unit cube $(0,1)^d$ for deep norm constrained  networks, with the error measured in the $W^{1,p}$-norm.
To be concrete, let $\phi : \mathbb{R}^d \to \mathbb{R}$ be a function computed by a ReLU neural network with width $W$ and depth $L$. In the $\ell$-th layer, the neural network computes an affine transformation $T_\ell(x) = A_\ell x + b_\ell$ and then applies the ReLU activation function element-wise (no activation in the output layer). For any $K\geq 0$, we define the norm constrained neural network $\mathcal{NN}(W, L, K)$ as the set of functions $\phi$ that satisfy the path norm constraint on the weights
$\|A_L\| \prod_{\ell=0}^{L-1} \max\left\{\|(A_\ell, b_\ell)\|, 1 \right\} \le K$.
Our main contributions are summarized as follows.
\begin{enumerate}[label=\textnormal{(\arabic*)}]
\item  For shallow norm constrained networks $\mathcal{NN}(W,1,K)$ with $L=1$ and bounded inner weights, we obtain the following  approximation
rates:
\[
\sup_{\|f\|_{W^{n,p}(\mathbb{B}^d)}\leq 1} \inf_{f_W \in \mathcal{NN}(W,1,K)} \| f - f_W \|_{W^{1,p}(\mathbb{B}^d)}
\le C \max\left\{ W^{-\frac{n-1}{d}}, K^{-\frac{n-1}{s-n}} \right\},
\]
where $d,n\in\mathbb{N}_{\ge 2}$, with $n < s= (d+3)/2$, and $2\le p\le\infty$, for some constant $C>0$.
This result extends the $L^\infty$-norm results of \citep{yang2025optimal} for shallow norm constrained ReLU networks to the $W^{1,p}$-norm setting, and also extends the $L^p$-norm results of \citep{mao2026approximation} to the $W^{1,p}$-norm by explicitly incorporating the path norm constraint.

\item For deep norm constrained networks $\mathcal{NN}(W,L,K)$, provided $W\geq CK^{(2d+n)/(d+d/p+1)}$ and $L\geq\lceil \log_{2}(d+n-1)\rceil+1$, we obtain an approximation rate that depends only on the path norm constraint $K$ (with $K\geq C$):
\[
\sup_{\|f\|_{W^{n,p}((0,1)^d)}\leq 1} \inf_{\phi\in\mathcal{NN}(W,L,K)} \| f - \phi \|_{W^{1,p}((0,1)^d)}
\le C K^{-\frac{n-1}{d + d/p + 1}},
\]
where $d\in\NN$, $n\in\NN_{\ge 2}$, $1\le p\le\infty$, for some constant $C>0$.
This rate is independent of the width and depth, making it suitable for over-parameterized networks, where the number of parameters exceeds the sample size.
This result extends the $L^\infty$-norm results of \citep{jiao2023approximation} for deep norm constrained ReLU networks to the $W^{1,p}$-norm setting.
\end{enumerate}

The proofs for shallow and deep norm constrained networks follow different constructions. For shallow norm constrained networks, we follow \citep{mao2026approximation} to construct an approximation of
$f$ via localized convolution sums, which is then realized by shallow norm-constrained neural networks.
Our approach is also motivated by the ideas in \citep{yang2025optimal,siegel2025optimal}.
For deep norm constrained networks,
we construct localized polynomial approximations
of  $f$  via a localized Taylor expansion and partitions of unity,
and then realize these approximations  by deep norm constrained networks, following the constructions in \citep{jiao2023approximation, guhring2020error,guhring2021approximation,
lu2021deep, hon2022simultaneous, jiao2025drm, jiao2024convergence}.
The main  technical challenge in both cases is to control the $W^{1,p}$-norm of the error.

The rest of the paper is organized as follows. In Section \ref{ApproximashallowReLU}, we introduce the neural networks with norm constraints and the Sobolev spaces used in this paper, and present our main approximation results for these networks.
The proofs of Theorem \ref{Soboappmainresult} and Theorem \ref{upperboundapp}
are given in Sections \ref{Proofofresult1} and \ref{ApproximadeepReLU}, respectively.
Section \ref{Conclusions} concludes this paper with a short
discussion. Finally, some useful auxiliary definitions and lemmas are provided in the Appendix.

\subsection{Notation}
Let us summarize all basic notations used throughout this paper.

The sets of natural numbers and real
numbers are denoted by $\NN$ and $\RR$, respectively. Furthermore, $\NN_{0}=\{0\}\cup\NN$ denotes the set
of non-negative integers.
For $k\in\NN_{0}$ we define $\NN_{\geq k} :=\{k, k+1, \ldots\}$.
For a set $A$
we denote its cardinality by $|A|\in\NN\cup\{\infty\}$.
If $x\in\RR$, then we write $\lceil x\rceil:= \min\{k\in\ZZ: k\geq x\}$,
where $\ZZ$ is the set of integers.

If $d\in\NN$  and $\|\cdot\|$ is a norm on $\RR^{d}$, then we denote for $x\in\RR^{d}$ and $r>0$ by $B_{r,\|\cdot\|}(x)$
the open ball around $x$ in $\RR^{d}$ with radius $r$, where the distance is measured in $\|\cdot\|$. By $\|x\|_{1}, \|x\|_{2} $ and
$\|x\|_{\infty}$
we denote the $\ell_{1}$-norm, the Euclidean norm and the maximum norm of $x$, respectively.

If $f : X\rightarrow Y$ and $g : Y\rightarrow Z$ are two functions, then we write
$g\circ f : X \rightarrow Z$ for their composition.
We use the usual multi-index notation, i.e., for $\alpha=(\alpha_{1},\alpha_{2},\ldots,\alpha_{d})\in\NN_{0}^{d}$
we write
$\|\alpha\|_{1}=\alpha_{1}+\ldots+\alpha_{d}$. Moreover, if $x = (x_{1},\ldots, x_{d}) \in \RR^{d}$, then we have
$$x^{\alpha} :=\prod_{i=1}^{d}x_{i}^{\alpha_{i}}.$$

Let $\Omega\subset\RR^{d}$ be open.
For a function $f : \Omega\rightarrow\RR$, we denote by
$$D^{\alpha}f :=\frac{\partial^{\|\alpha\|_{1}}f}{\partial x_{1}^{\alpha_{1}}\partial x_{2}^{\alpha_{2}}\cdots \partial x_{d}^{\alpha_{d}}}$$
its (weak or classical) derivative of order $\alpha$.
We denote by $L^{p}(\Omega)$, $1 \leq p \leq \infty$, the standard Lebesgue spaces, and by $C_c^\infty(\mathbb{R}^d)$ the space of smooth compactly supported functions on $\mathbb{R}^d$, i.e.,
\[
C_c^\infty(\mathbb{R}^d) := \{ f \in C^\infty(\mathbb{R}^d) \mid \operatorname{supp} f \subset\subset \mathbb{R}^d \}.
\]

We use $X \lesssim Y$ or $Y \gtrsim X$ to denote the statement that $X \le C Y$ for some constant $C > 0$. We denote $X \asymp Y$ when $X \lesssim Y \lesssim X$.

For any $\epsilon>0$, we define the open and closed balls of radius $\epsilon$ as
\[
\mathbb{B}_\epsilon^d := \{x\in\mathbb{R}^d : \|x\|_2 < \epsilon\}\quad \text{and}\quad
\overline{\mathbb{B}_\epsilon^d} := \{x\in\mathbb{R}^d : \|x\|_2 \le \epsilon\}.
\]
When $\epsilon=1$, we simply write $\mathbb{B}^d$ and $\overline{\mathbb{B}^d}$ for the open and closed unit balls, respectively. For example, $\mathbb{B}_2^d$ and $\overline{\mathbb{B}_2^d}$ denote the open and closed balls of radius $2$.
We denote by $\mathbb{S}^{d-1}$ the unit sphere in $\mathbb{R}^{d}$, i.e., $\mathbb{S}^{d-1} := \{x\in\mathbb{R}^{d} : \|x\|_2 = 1\}$.

We remark that throughout this paper, we will use $C$ to denote a generic
constant which may change from line to line. We will indicate the dependence of this constant whenever it must be clarified.

\section{Approximation rates for norm constrained  neural networks}\label{ApproximashallowReLU}
In this section, we first introduce the basic terminology of neural networks with norm constraints that will be used
in this paper, and adopt the corresponding notation from \citep{jiao2023approximation, mao2026approximation, siegel2025optimal}.

Let $L,d,m, N_1, \ldots, N_L \in \NN$. We consider the function $\phi : \RR^d \rightarrow\RR^m$ that can be parameterized by a ReLU neural network of the form
\begin{equation}\label{definenet}
\begin{aligned}
\phi_0(x) &= x, \\
\phi_{\ell+1}(x) &= \sigma(A_\ell \phi_\ell(x) + b_\ell), \quad \ell = 0, \ldots, L-1, \\
\phi(x) &= A_L \phi_L(x),
\end{aligned}
\end{equation}
where $A_\ell \in \RR^{N_{\ell+1} \times N_\ell}$, $b_\ell \in \RR^{N_{\ell+1}}$ with $N_0 = d$ and $N_{L+1} = m$. The activation function $\sigma(x) := \max\{0,x\}$ is the Rectified Linear Unit function (ReLU) and it is applied element-wise. The numbers $W := \max\{N_1, \ldots, N_L\}$ and $L$ are called the width and depth of the neural network, respectively.
We denote by $\mathcal{NN}_{d,m}(W,L)$ the set of functions that can be parameterized by ReLU neural networks with width $W$ and depth $L$. When the input dimension $d$ and output dimension $m$ are clear from the context, we simply denote it by $\mathcal{NN}(W,L)$.
For any $K\geq 0$, we define the norm constrained neural network $\mathcal{NN}(W, L, K)$ as the set of functions $\phi\in \mathcal{NN}(W, L)$ of the form (\ref{definenet}) that satisfy the following path norm constraint on the weights
\begin{equation}\label{definenet4}
\|A_L\| \prod_{\ell=0}^{L-1} \max\left\{\|(A_\ell, b_\ell)\|, 1 \right\} \le K,
\end{equation}
where $(A_\ell, b_\ell)$ denotes the matrix obtained by appending $b_\ell$ as an extra column to $A_\ell$, and $\|A\|$ is some norm of a matrix $A=(a_{ij}) \in \RR^{m \times n}$.
We only consider the operator norm defined by $\|A\| := \sup_{\|x\|_\infty \le 1} \|Ax\|_\infty$ in this paper, which means that $\|A\|$ is the maximum 1-norm of the rows of $A$, i.e.,
\[
\|A\| = \max_{1 \le i \le m} \sum_{j=1}^n |a_{ij}|.
\]
For a neural network of the form (\ref{definenet}) with depth $L=1$ and  output dimension $m=1$, we can write it as
\[
\sum_{i=1}^W a_i \sigma(\omega_i\cdot x + b_i),
\]
where $a_i, b_i\in\mathbb{R}$, $\omega_i\in\mathbb{R}^d$,
i.e., a shallow ReLU  neural network with  width $W$.
Moreover, the set of shallow ReLU  neural networks is
\[
\Sigma_W(\mathbb{R}_{1}^d) := \left\{ \sum_{i=1}^W a_i \sigma(\omega_i \cdot x + b_i) :  a_i, b_i \in \mathbb{R}, \omega_i \in \mathbb{R}^d \right\}.
\]
For $x\in\mathbb{B}^d$, we introduce the dictionary
\begin{equation}\label{definenet5}
\PP_1^d := \{\sigma(\omega \cdot x + b) : \omega \in \mathbb{S}^{d-1}, b \in [-1,1]\},
\end{equation}
which consists of all possible outputs of a single neuron with bounded inner weights.
We therefore define the associated  set
\[
\Sigma_W(\PP_1^d) := \left\{ \sum_{i=1}^W a_i \eta_i : a_i \in \mathbb{R}, \eta_i \in \PP_1^d \right\},
\]
which is a subset  of the set of shallow ReLU  neural networks.
Consequently, the path norm constraint on the weights in (\ref{definenet4}) becomes $\sum_{i=1}^W |a_i| \le K$ (up to an absolute constant).
We denote by $\Sigma_W^{K}(\PP_1^d)$ the set
\[
\left\{ \sum_{i=1}^W a_i \eta_i : a_i \in \mathbb{R}, \eta_i \in \PP_1^d,\ \sum_{i=1}^W |a_i| \le K \right\}
\]
with a constraint $K$ on the output weights.
Thus, $\Sigma_W^{K}(\PP_1^d)$ is the shallow special case of $\mathcal{NN}(W,L,K)$ with depth $L=1$ and output dimension $1$.

We now recall the definition of Sobolev spaces, which are central objects in analysis and the theory of PDEs (see, e.g., \citep{adams2003sobolev,evans2010partial,maz2013sobolev}).

\begin{definition}\label{Def:Sobolevspaces1}
Assume that $\Omega$ is an open subset of $\mathbb{R}^d$, and let $n\in\mathbb{N}$, $1\le p\le\infty$. The Sobolev space $W^{n,p}(\Omega)$ consists of functions $f\in L^p(\Omega)$ such that for every multi-index $\alpha\in\mathbb{N}_0^d$ with $\|\alpha\|_1\le n$, the weak derivative $D^\alpha f$ exists and belongs to $L^p(\Omega)$. Thus
\[
W^{n,p}(\Omega) := \{ f\in L^p(\Omega) : D^\alpha f\in L^p(\Omega) \text{ for all } \alpha\in\mathbb{N}_0^d  \text{ with } \|\alpha\|_1\le n \}.
\]
For $f\in W^{n,p}(\Omega)$ and $1\le p<\infty$, we define the norm
\begin{equation}\label{eq:sobolevdefi1}
\|f\|_{W^{n,p}(\Omega)} := \left( \sum_{0\le\|\alpha\|_1\le n} \|D^\alpha f\|_{L^p(\Omega)}^p \right)^{1/p},
\end{equation}
and
\begin{equation}\label{eq:sobolevdefi2}
\|f\|_{W^{n,\infty}(\Omega)} := \max_{0\le\|\alpha\|_1\le n} \|D^\alpha f\|_{L^\infty(\Omega)}.
\end{equation}
Additionally, for $0\leq k\leq n$, on $W^{n,p}(\Omega)$ we introduce the family of semi-norms
$$|f|_{ W^{k,p}(\Omega)} := \left(\sum_{\|\alpha\|_{1}=k}\|D^{\alpha}f\|_{ L^{p}(\Omega)}^{p}\right)^{1/p}\quad\hbox{and}\quad
|f|_{ W^{k,\infty}(\Omega)} :=\max_{\|\alpha\|_{1}=k}\|D^{\alpha}f\|_{ L^{\infty}(\Omega)},$$
respectively.
\end{definition}

We remark that (fractional) Sobolev spaces can be defined for non-integral $\alpha$ (see \citep{di2012hitchhiker's}).
For fractional Sobolev spaces we will only need the Hilbert space case $p=2$. In this case, it is well known that if the domain is all of $\mathbb{R}^d$, the Sobolev norm can be characterized via the Fourier transform: for any $s\in\mathbb{R}$,
\[
\| f \|_{W^{s,2}(\mathbb{R}^d)}^2 \asymp \int_{\mathbb{R}^d} (1+\|\xi\|_{2})^{2s} |\widehat{f}(\xi)|^2 d\xi,
\]
with seminorm
\[
| f |_{W^{s,2}(\mathbb{R}^d)}^2 \asymp \int_{\mathbb{R}^d} \|\xi\|_{2}^{2s} |\widehat{f}(\xi)|^2 d\xi,
\]
where $\widehat{f}(\xi) :=\int_{\mathbb{R}^d} e^{-i\xi\cdot x}f(x)\,dx$ is the Fourier transform of $f$. For a bounded domain $\Omega\subset\mathbb{R}^d$, we define the fractional Sobolev norm by restriction:
\begin{equation}\label{eq:sobolevextention}
\| f \|_{W^{s,2}(\Omega)} := \inf\bigl\{ \| \widetilde{f} \|_{W^{s,2}(\mathbb{R}^d)} : \widetilde{f}(x)=f(x) \text{ on } \Omega \bigr\}.
\end{equation}
It is known that the definition in (\ref{eq:sobolevextention}) is equivalent to other definitions of the fractional Sobolev spaces \citep{di2012hitchhiker's}.
In this paper, we will avoid these technicalities and simply adopt (\ref{eq:sobolevextention}) as our definition for the case $p=2$.
Note that by the well-known Sobolev extension theory (see, e.g., \citep{adams2003sobolev,evans2010partial,maz2013sobolev, stein1970singular}), this definition coincides with the Sobolev norm defined in (\ref{eq:sobolevdefi1}) when $s$ is an integer and $p=2$.

\begin{remark}\label{Rem: extentheorem}
Let $\Omega$ be $\BB^{d}$ or $(0,1)^{d}$. By the standard Sobolev extension theorems (see, e.g., \citep{adams2003sobolev, evans2010partial, maz2013sobolev,stein1970singular}), there exists an extension operator $E : W^{n,p}(\Omega) \to W^{n,p}(\mathbb{R}^d)$. For any $f\in W^{n,p}(\Omega)$, set $\widetilde{f} := Ef$,
and the following estimate holds:
\begin{equation}\label{eq:extentheorem}
\|\widetilde{f}\|_{W^{n,p}(\mathbb{R}^d)} \le C \|f\|_{W^{n,p}(\Omega)},
\end{equation}
where $C = C(d, n, p)$ is the norm of the extension operator.
\end{remark}

We are interested in approximating functions in the unit ball of the Sobolev space $W^{n,p}$ with smoothness index $n\in\mathbb{N}_{\ge 2}$ by both shallow and deep norm constrained neural networks, with the error measured in the $W^{1,p}$-norm.
The unit ball is defined as
\begin{equation}\label{Sobolevuni1}
\mathcal{F}_{d,n,p} := \{f\in W^{n,p}(\Omega) : \|f\|_{W^{n,p}(\Omega)}\leq 1\}.
\end{equation}
where $\Omega$ is $\BB^{d}$ or $(0,1)^{d}$.

Our main results can be summarized in the following theorems.
To approximate functions in $\mathcal{F}_{d,n,p}$ by shallow norm constrained  neural networks,
we choose $\Omega=\BB^{d}$ in (\ref{Sobolevuni1}) and obtain the following approximate rates.
\begin{theorem}\label{Soboappmainresult}
Let $d,n\in\mathbb{N}_{\ge 2}$, $s = (d + 3)/2$ with $n< s$, and $2\le p\le\infty$.
Then there exists a constant $C=C(d,n,p)>0$
such that
\begin{equation}\label{mainresult}
\sup_{f\in\mathcal{F}_{d,n,p}}\inf_{f_W \in \Sigma_W^{K}(\PP_1^{d})} \| f - f_W \|_{W^{1,p}(\BB^{d})} \le C\max\left\{W^{-\frac{n-1}{d}},K^{-\frac{n-1}{s-n}}\right\}.
\end{equation}
Thus, the rate $\mathcal{O}(W^{-(n-1)/d})$ holds when $K \ge C W^{(s-n)/d}$.
\end{theorem}

For approximating functions in $\mathcal{F}_{d,n,p}$ by deep norm constrained  neural networks,
we choose $\Omega=(0,1)^{d}$ in (\ref{Sobolevuni1}) and obtain the following approximate rate.
\begin{theorem}\label{upperboundapp}
Let $d\in\NN, n\in\NN_{\geq 2}$, $1\leq p\leq\infty$. Then there exists a constant $C=C(d,n,p)>0$ such that for any $K\geq C$,
any $W\geq CK^{(2d+n)/(d+d/p+1)}$ and $L\geq\lceil \log_{2}(d+n-1)\rceil+1,$
\begin{equation}\label{mainresult2}
\sup_{f\in\mathcal{F}_{d,n,p}}\inf_{\phi\in\mathcal{NN}(W,L,K)}
\left\|f-\phi\right\|_{W^{1,p}((0,1)^{d})}
\leq CK^{-\frac{n-1}{d+d/p+1}}.
\end{equation}
\end{theorem}

We now compare the approximation error bounds $\mathcal{E}_{W,K}$ and $\mathcal{E}_{K}$ provided by Theorem \ref{Soboappmainresult} and
Theorem \ref{upperboundapp}, respectively.

First, Theorem \ref{Soboappmainresult} applies to shallow norm constrained networks of the form $\Sigma_W^{K}(\PP_1^d)$, whereas Theorem \ref{upperboundapp} deals with deep norm constrained networks $\mathcal{NN}(W,L,K)$. The former provides the following error bound:
\[
\mathcal{E}_{W,K} \lesssim \max\left\{W^{-\frac{n-1}{d}}, K^{-\frac{n-1}{s-n}}\right\}.
\]
In contrast, the latter gives a error bound depending only on $K$:
\[
\mathcal{E}_{K} \lesssim K^{-\frac{n-1}{d+d/p+1}},
\]
provided the width $W$ and depth $L$ are sufficiently large ($W$ needs to grow with $K$). This makes Theorem \ref{upperboundapp} particularly suitable for over-parameterized deep networks, where the width can be chosen sufficiently large without affecting the rate.

Second, the two results differ in both their parameter ranges and their dependence on the path norm constraint $K$. Theorem \ref{Soboappmainresult} is restricted to the regime $n < s = (d+3)/2$ and $2\le p\le\infty$, while Theorem \ref{upperboundapp} applies to the broader class $n\ge 2$ and $1\le p\le\infty$. By contrast, the shallow-network error bound $\mathcal{E}_{W,K}$ in $K$, namely $\mathcal{O}(K^{-(n-1)/(s-n)})$, is stronger than the deep-network error bound $\mathcal{E}_{K}$, namely $\mathcal{O}(K^{-(n-1)/(d+d/p+1)})$, since $s-n < d + d/p + 1$ for all admissible parameters.

\begin{remark}\label{Theoremimple}
Theorem \ref{Soboappmainresult} requires \(d\ge 2\) and \(p\ge 2\).
The condition $d\ge 2$ is necessary: if $d=1$, then $s=2$, and no $n\in\NN_{\ge 2}$ satisfies $n<s$.
The restriction $p\ge 2$ is also essential. Indeed, for $1\le p<2$, the proof of Theorem \ref{Soboappmainresult} in Section \ref{Proofofresult1} would require the embedding
\[
W^{s,1}(\mathbb{B}^d) \subset \mathcal{K}_{1}(\PP^{d}_{1})
\]
with $s=(d+3)/2$ (for the variation space $\mathcal{K}_{1}(\PP^{d}_{1})$, see (\ref{eq:convexhullsobolev}), (\ref{variationspace}), and (\ref{variationspace1})). However, this embedding cannot hold. Since $\mathcal{K}_{1}(\PP^{d}_{1}) \subset W^{1,\infty}(\mathbb{B}^d)$, it would imply
\[
W^{s,1}(\mathbb{B}^d) \subset W^{1,\infty}(\mathbb{B}^d),
\]
which, by Sobolev embedding theory, would require $s\ge d+1$. This contradicts $s=(d+3)/2$, as $d+1 > (d+3)/2$ for all $d\ge 2$. Hence, for
$1\le p<2$, the present method cannot yield the desired rates.
We note that the exponent $s=(d+3)/2$ in Lemma \ref{Sobolevembed} is sharp in the sense of metric entropy (see \citep[Theorem 1.1]{mao2026approximation}). Consequently, Theorem \ref{Soboappmainresult} only covers the regime $n < (d+3)/2$.
Extending the results for shallow norm constrained  networks to the cases $n \ge (d+3)/2$ or $1\le p<2$
and whether deep norm constrained networks can achieve even better approximation rates remain interesting open problems.
\end{remark}

\section{Proof of Theorem \ref{Soboappmainresult}}\label{Proofofresult1}

This section is devoted to the proof of Theorem \ref{Soboappmainresult}.
We adopt the strategy  in the proof of \citep[Corollary 1.3]{mao2026approximation}.
Following their construction, we first construct, via localized convolution sums, an approximating function $\widetilde{f}_{\epsilon,\delta}$ of $f$, which is then approximated by shallow norm constrained neural networks. Our approach to the proof is also motivated by the ideas in \citep{yang2025optimal,siegel2025optimal}.

The following lemma establishes the rate of approximation of functions in $W^{n,p}(\mathbb{B}^d)$ by $\widetilde{f}_{\epsilon,\delta}$ in the $W^{1,p}$-norm. It is a key ingredient in the proof of Theorem \ref{Soboappmainresult}.

\begin{lemma}\label{Le:SmooFuncAppro}
Let $d\in\mathbb{N}$, $n\in\mathbb{N}_{\ge 2}$, $1\le p\le\infty$, and let $\phi$ be a normalized smooth radially symmetric bump function supported on $\overline{\BB^d}$ as in Definition~\ref{Def:smooradsymbu}.
For any $f\in W^{n,p}(\BB^d)$, let $\widetilde{f}\in W^{n,p}(\mathbb{R}^d)$ be an extension of $f$ with $\operatorname{supp}\widetilde{f}\subseteq\overline{\BB_2^d}$ (or any other fixed domain containing $\overline{\BB^d}$). For any $\delta>0$, set $\phi_\delta(x)=\delta^{-d}\phi(x/\delta)$ for  $x\in\mathbb{R}^d$. Fix $\epsilon\in(0,1]$ and let
\begin{equation}\label{approxfunc}
\widetilde{f}_{\epsilon,\delta}(x) :=\sum_{j=1}^{n-1}\binom{n-1}{j}(-1)^{j-1}\int_{\mathbb{R}^d}\phi_\epsilon(\theta)(\phi_\delta*\widetilde{f})(x-j\theta)d\theta,\qquad x\in\mathbb{R}^d.
\end{equation}
Then there exists a constant $C=C(d,n,p)>0$
such that for all sufficiently small $\delta>0$  (depending on
$\epsilon$ and $f$),
\[
\|f - \widetilde{f}_{\epsilon,\delta}|_{\mathbb{B}^d}\|_{W^{1,p}(\BB^d)} \le C\epsilon^{n-1}\|f\|_{W^{n,p}(\mathbb{B}^d)},
\]
where $\widetilde{f}_{\epsilon,\delta}|_{\mathbb{B}^d}$ denotes the restriction of $\widetilde{f}_{\epsilon,\delta}$ on $\mathbb{B}^d$.
\end{lemma}

\begin{proof}
By the definition of $\phi$, for any $\delta>0$, the function $\phi_\delta(x) = \delta^{-d} \phi(x/\delta)$
satisfies $\phi_{\delta}\in C_{c}^{\infty}(\RR^{d})$, $\phi_{\delta}\ge 0$, $\operatorname{supp} \phi_{\delta} = \overline{\BB_{\delta}^{d}}$ and \[
\|\phi_{\delta}\|_{L^{1}(\RR^{d})}=\int_{\RR^{d}} \phi_{\delta}(x)dx=\int_{\RR^{d}} \phi(x)dx=1.
\]
since  $\phi_{\delta}\in C_{c}^{\infty}(\RR^{d})$ and $\widetilde{f}\in W^{n,p}(\mathbb{R}^d)$ with $\operatorname{supp} \widetilde{f} \subseteq\overline{\BB_{2}^d}$,
 by  the properties of convolution,
we get that $\phi_{\delta}\ast \widetilde{f}\in C_{c}^{\infty}(\RR^{d})$,
and for any multi-index $\alpha\in\NN_{0}^{d}$ with $\|\alpha\|_{1} \le n$,
\[
D^{\alpha}(\phi_{\delta}\ast \widetilde{f}) = \phi_{\delta} \ast D^{\alpha}\widetilde{f}.
\]
Hence,
for any $1\leq p\leq\infty$ and  any $\alpha\in\NN_{0}^{d}$ with $\|\alpha\|_{1}\leq 1$, we have
\begin{align}\label{smoothLpappsob}
&\bigl\|D^{\alpha}(\phi_{\delta}\ast\widetilde{f}-\widetilde{f})\bigr\|_{L^{p}(\RR^{d})}\nonumber\\
&\quad=\bigl\|\phi_{\delta}\ast D^{\alpha}\widetilde{f} - D^{\alpha}\widetilde{f}\bigr\|_{L^{p}(\RR^{d})}\nonumber\\
&\quad = \left\|\int_{\RR^{d}} \phi_{\delta}(y) \bigl( D^{\alpha}\widetilde{f}(x-y) - D^{\alpha}\widetilde{f}(x) \bigr)dy\right\|_{L^{p}(\RR^{d};dx)}\nonumber\\
&\quad\leq \int_{\RR^{d}} \phi_{\delta}(y) \bigl\|D^{\alpha}\widetilde{f}(x-y) - D^{\alpha}\widetilde{f}(x)\bigr\|_{L^{p}(\RR^{d};dx)}dy\nonumber\\
&\quad= \int_{\overline{\BB_{\delta}^{d}}} \phi_{\delta}(y) \bigl\|D^{\alpha}\widetilde{f}(x-y) - D^{\alpha}\widetilde{f}(x)\bigr\|_{L^{p}(\RR^{d};dx)}dy,
\end{align}
where  the second step uses $\int_{\RR^{d}} \phi_{\delta}(y)dy=1$,
the third step uses  Minkowski's integral inequality, and the last step uses the fact that $\operatorname{supp} \phi_{\delta} = \overline{\BB_{\delta}^{d}}$. For any $1\leq p< \infty$ and any $\alpha\in\NN_{0}^{d}$ with $\|\alpha\|_{1}\leq 1$,
we have $D^{\alpha}\widetilde{f}\in L^{p}(\RR^{d})$, and since the translation operator is continuous in $L^{p}(\RR^{d})$,
\begin{equation}\label{smoottracon}
\lim_{\|y\|_{2} \rightarrow 0} \bigl\| D^{\alpha}\widetilde{f}(x - y) - D^{\alpha}\widetilde{f}(x) \bigr\|_{L^p(\RR^{d};dx)} = 0.
\end{equation}
Using (\ref{smoothLpappsob}) and (\ref{smoottracon}) together with $\int_{\overline{\BB_{\delta}^{d}}} \phi_{\delta}(y)dy=1$,
we get
\begin{equation}\label{smoothLpapp}
\lim_{\delta\rightarrow 0}\bigl\|D^{\alpha}(\phi_{\delta}\ast \widetilde{f} - \widetilde{f})\bigr\|_{L^{p}(\RR^{d})}=0.
\end{equation}
Since (\ref{smoothLpapp}) holds for all $\alpha\in\NN_{0}^{d}$ with $\|\alpha\|_{1}\leq 1$, the definition of  the $W^{1,p}$-norm yields
\begin{equation}\label{smoothLpapp1}
\lim_{\delta\rightarrow 0}\bigl\|\phi_{\delta}\ast \widetilde{f} - \widetilde{f}\bigr\|_{W^{1,p}(\RR^{d})}=0.
\end{equation}
When $p=\infty$,
since $\widetilde{f}\in W^{n,\infty}(\RR^{d})$ with $n\geq 2$, both $\widetilde{f}$ and its first-order partial derivatives
$D^{\alpha}\widetilde{f}$ ($\|\alpha\|_{1}=1$) are uniformly continuous.
Hence, for any $\alpha\in\NN_{0}^{d}$ with $\|\alpha\|_{1}\leq 1$,
\begin{equation}\label{smoothLpappinf}
\lim_{\|y\|_{2} \rightarrow 0} \bigl\| D^{\alpha}\widetilde{f}(x - y) - D^{\alpha}\widetilde{f}(x) \bigr\|_{L^\infty(\RR^{d};dx)} = 0.
\end{equation}
Applying  (\ref{smoothLpappsob}) and (\ref{smoothLpappinf}) together with $\int_{\BB_{\delta}^{d}} \phi_{\delta}(y)dy=1$, we obtain
\begin{equation}\label{smoothLpapp2}
\lim_{\delta\rightarrow 0}\bigl\|D^{\alpha}(\phi_{\delta}\ast \widetilde{f} -\widetilde{f})\bigr\|_{L^{\infty}(\RR^{d})}=0.
\end{equation}
Since (\ref{smoothLpapp2}) holds for all $\alpha\in\NN_{0}^{d}$ with $\|\alpha\|_{1}\leq 1$,
the definition of the $W^{1,\infty}$-norm gives
\begin{equation}\label{smoothLpapp3}
\lim_{\delta\rightarrow 0}\bigl\|\phi_{\delta}\ast \widetilde{f} - \widetilde{f}\bigr\|_{W^{1,\infty}(\RR^{d})}=0.
\end{equation}
Furthermore, for any $1\leq p\leq\infty$ and any $\alpha\in\NN_{0}^{d}$ with $\|\alpha\|_{1}\leq n$,
by Young's inequality,  we have
\[
\bigl\|D^\alpha (\phi_\delta \ast \widetilde{f})\bigr\|_{L^{p}(\RR^{d})} = \bigl\|\phi_\delta \ast (D^\alpha \widetilde{f})\bigr\|_{L^{p}(\RR^{d})} \leq \bigl\|D^\alpha \widetilde{f}\bigr\|_{L^{p}(\RR^{d})}\|\phi_{\delta}\|_{L^{1}(\RR^{d})} = \bigl\|D^\alpha \widetilde{f}\bigr\|_{L^{p}(\RR^{d})}.
\]
Therefore, for any $1\leq p\leq\infty$ and any $\delta>0$, by the definition of the $W^{n,p}$-norm, we get
\begin{equation}\label{smoothLpapp4}
\bigl\|\phi_\delta \ast\widetilde{f}\bigr\|_{W^{n,p}(\RR^{d})} \leq \|\widetilde{f}\|_{W^{n,p}(\RR^{d})}.
\end{equation}

Now, fix $\epsilon\in (0,1]$, for any $\delta>0$, since $\phi_\epsilon$ and $\phi_\delta \ast \widetilde{f}$ are both in $C_c^\infty(\mathbb{R}^d)$, each term $\int_{\RR^{d}} \phi_\epsilon(\theta) (\phi_\delta \ast \widetilde{f})(x - j\theta)d\theta$ is a $C_c^\infty$ function of
$x$ (smoothness follows by differentiation under the integral, and compact support is inherited from the integrand). Thus, $\widetilde{f}_{\epsilon,\delta} \in C_c^\infty(\mathbb{R}^d)$. Furthermore,
for any $1\leq p\leq\infty$ and  any $\alpha\in\NN_{0}^{d}$ with $\|\alpha\|_{1}\leq 1$, we get
\begin{align}\label{smoothappint}
&\bigl\|D^{\alpha}(\phi_\delta \ast \widetilde{f} - \widetilde{f}_{\epsilon,\delta}) \bigr\|_{L^p(\RR^{d})}\nonumber\\
&\quad=\left\|\int_{\RR^{d}} \phi_{\epsilon}(\theta)\left(D^{\alpha}(\phi_\delta \ast \widetilde{f})(x)-\sum_{j=1}^{n-1} \binom{n-1}{j} (-1)^{j-1}D^{\alpha}(\phi_\delta \ast \widetilde{f})(x - j\theta)\right)d\theta\right\|_{L^p(\RR^{d};dx)}\nonumber\\
&\quad=\left\|\int_{\RR^{d}} \phi_{\epsilon}(\theta)\left(D^{\alpha}(\phi_\delta \ast \widetilde{f})(x)+\sum_{j=1}^{n-1} \binom{n-1}{j} (-1)^{j}D^{\alpha}(\phi_\delta \ast \widetilde{f})(x - j\theta)\right)d\theta\right\|_{L^{p}(\RR^{d};dx)}\nonumber\\
&\quad=\left\|\int_{\RR^{d}} \phi_\epsilon(\theta)\left(\sum_{j=0}^{n-1} \binom{n-1}{j} (-1)^{j}D^{\alpha}(\phi_\delta \ast \widetilde{f})(x - j\theta)\right)d\theta\right\|_{L^p(\RR^{d};dx)}\nonumber\\
&\quad\le \int_{\RR^{d}} \phi_\epsilon(\theta)
\left\| \sum_{j=0}^{n-1} \binom{n-1}{j} (-1)^j D^{\alpha}(\phi_\delta \ast \widetilde{f})(x - j\theta) \right\|_{L^p(\RR^{d};dx)}d\theta,
\end{align}
where  the first step uses $\int_{\RR^{d}} \phi_{\epsilon}(\theta)d\theta=1$ and
the last step uses  Minkowski's integral inequality.
Next, fix a $\theta \in \mathbb{R}^d$. For any function $h: \RR^{d}\to\RR$, define the first-order finite difference operator $\Delta_\theta$ by $\Delta_\theta h(x)=h(x)-h(x-\theta)$, and let $\Delta_\theta^{k+1}=\Delta_\theta\circ\Delta_\theta^{k}$ for $k\in\NN$. Then the $(n-1)$-th order difference of $h$ can be expressed as
\begin{equation}\label{funcfinidffidef}
\Delta_\theta^{n-1} h(x) = \sum_{j=0}^{n-1} \binom{n-1}{j} (-1)^j h(x-j\theta).
\end{equation}
Applying (\ref{funcfinidffidef}) to $h = D^{\alpha}(\phi_\delta \ast \widetilde{f})$, we obtain
\begin{equation}\label{finidffidef}
\left\| \sum_{j=0}^{n-1} \binom{n-1}{j} (-1)^j D^{\alpha}(\phi_\delta \ast \widetilde{f})(x - j\theta) \right\|_{L^p(\mathbb{R}^d;dx)} = \bigl\| \Delta_\theta^{n-1} D^{\alpha}(\phi_\delta \ast \widetilde{f}) \bigr\|_{L^p(\mathbb{R}^d)}.
\end{equation}
By the fundamental theorem of calculus, we get
\begin{equation}\label{diffopeinte}
\Delta_\theta  D^{\alpha}(\phi_\delta \ast \widetilde{f})(x)= -\int_0^1 \frac{d}{dt}D^{\alpha}(\phi_\delta \ast \widetilde{f})(x-t\theta)dt = \int_0^1 \nabla D^{\alpha}(\phi_\delta \ast \widetilde{f})(x-t\theta) \cdot \theta dt.
\end{equation}
Since the operator $\Delta_\theta$ (applied to a function of $x$) commutes with integration over $t$, we may apply (\ref{diffopeinte}) repeatedly. For $n=2$, using linearity and the fact that $\Delta_\theta$ acts on the variable $x$, we have
\[
\begin{aligned}
&\Delta_\theta^2  D^{\alpha}(\phi_\delta \ast \widetilde{f})(x) = \Delta_\theta(\Delta_\theta  D^{\alpha}(\phi_\delta \ast \widetilde{f}))(x)\\
&\quad = \Delta_\theta\left(\int_0^1 \nabla  D^{\alpha}(\phi_\delta \ast \widetilde{f})(x - t_1 \theta) \cdot \theta dt_1\right)\\
&\quad=\int_0^1 \Delta_\theta\bigl(\nabla D^{\alpha}(\phi_\delta \ast \widetilde{f})(x - t_1 \theta) \cdot \theta\bigr) dt_1.
\end{aligned}
\]
Now apply (\ref{diffopeinte}) to the function $(\phi_\delta \ast \widetilde{f})_{t_1}(x) = \nabla  D^{\alpha}(\phi_\delta \ast \widetilde{f})(x - t_1 \theta) \cdot \theta$. Because
$\Delta_\theta (\phi_\delta \ast \widetilde{f})_{t_1}(x) = \int_0^1 \nabla (\phi_\delta \ast \widetilde{f})_{t_1}(x - t_2 \theta) \cdot \theta dt_2$
and $\nabla (\phi_\delta \ast \widetilde{f})_{t_1}(x) = D^2  D^{\alpha}(\phi_\delta \ast \widetilde{f})(x - t_1 \theta) \cdot \theta$, we obtain
\[
\Delta_\theta^2  D^{\alpha}(\phi_\delta \ast \widetilde{f})(x) =\int_0^1 \int_0^1 D^2  D^{\alpha}(\phi_\delta \ast \widetilde{f})\bigl(x - (t_1 + t_2) \theta\bigr) \cdot (\theta \otimes \theta) dt_1 dt_2.
\]
Proceeding inductively yields the general formula
\begin{align}\label{smoothappintmul}
\Delta_\theta^{n-1} D^{\alpha}(\phi_\delta \ast \widetilde{f})(x) &= \int_{[0,1]^{n-1}} D^{n-1} D^{\alpha}(\phi_\delta \ast \widetilde{f})\left(x - \left(\sum_{i=1}^{n-1} t_i\right) \theta\right) \cdot \theta^{\otimes (n-1)}dt_1 \cdots dt_{n-1}\nonumber\\
&=\int_{[0,1]^{n-1}} D^{n-1} D^{\alpha}(\phi_\delta \ast \widetilde{f})\bigl(x - (\boldsymbol{1}^{\mathsf{T}} t) \theta\bigr) \cdot \theta^{\otimes (n-1)}dt,
\end{align}
where $\boldsymbol{1}^{\mathsf{T}} t = t_1 + \cdots + t_{n-1}$, $\theta^{\otimes (n-1)} = \theta \otimes \cdots \otimes \theta$ ($(n-1)$ times), and $D^{n-1} D^{\alpha}(\phi_\delta \ast \widetilde{f}) \cdot \theta^{\otimes (n-1)}$ denotes the full contraction of the $(n-1)$-th derivative tensor
of $D^{\alpha}(\phi_\delta \ast \widetilde{f})$ with $\theta^{\otimes (n-1)}$
(this coincides with the
$(n-1)$-th derivative of $D^{\alpha}(\phi_\delta \ast \widetilde{f})$
in the direction $\theta$).
Taking absolute values in (\ref{smoothappintmul}) and applying the triangle inequality gives
\begin{align}\label{smoothappinequ}
\bigl|\Delta_\theta^{n-1} D^{\alpha}(\phi_\delta \ast \widetilde{f})(x)\bigr|
&\leq\int_{[0,1]^{n-1}} \bigl|D^{n-1} D^{\alpha}(\phi_\delta \ast \widetilde{f})\bigl(x - (\boldsymbol{1}^{\mathsf{T}} t) \theta\bigr) \cdot \theta^{\otimes (n-1)}\bigr|dt\nonumber\\
&\leq \|\theta\|_{1}^{n-1}\int_{[0,1]^{n-1}} \bigl\|D^{n-1} D^{\alpha}(\phi_\delta \ast \widetilde{f})\bigl(x - (\boldsymbol{1}^{\mathsf{T}} t) \theta\bigr) \bigr\|_{1}dt,
\end{align}
where $\bigl\|D^{n-1} D^{\alpha}(\phi_\delta \ast \widetilde{f})\bigl(x - (\boldsymbol{1}^{\mathsf{T}} t) \theta\bigr)\bigr\|_{1}$
denotes the $\ell_{1}$-norm of the tensor $D^{n-1} D^{\alpha}(\phi_\delta \ast \widetilde{f})\bigl(x - (\boldsymbol{1}^{\mathsf{T}} t) \theta\bigr)$
(viewed as a vector in $\RR^{d^{n-1}}$).
For $ p=\infty$ and any  $\alpha\in\NN_{0}^{d}$ with $\|\alpha\|_{1}\leq 1$,
(\ref{smoothappinequ}) already implies that
\begin{equation}\label{fismoothappinfi}
\| \Delta_\theta^{n-1} D^{\alpha}(\phi_\delta \ast \widetilde{f}) \|_{L^\infty(\RR^d)} \le C\|\theta\|_{1}^{n-1} \|\phi_\delta \ast \widetilde{f}\|_{W^{n,\infty}(\RR^d)},
\end{equation}
where $C=C(d,n)>0$ is a constant.
When $1\leq p<\infty$ and $\alpha\in\NN_{0}^{d}$ with $\|\alpha\|_{1}\leq 1$, we first apply Jensen's inequality to (\ref{smoothappinequ}). This gives the pointwise estimate
\begin{equation}\label{fismoothappinequ}
\bigl|\Delta_\theta^{n-1} D^{\alpha}(\phi_\delta \ast \widetilde{f})(x)\bigr|^{p}\leq \|\theta\|_{1}^{(n-1)p}\int_{[0,1]^{n-1}} \bigl\|D^{n-1} D^{\alpha}(\phi_\delta \ast \widetilde{f})\bigl(x - (\boldsymbol{1}^{\mathsf{T}} t) \theta\bigr) \bigr\|_{1}^{p}dt.
\end{equation}
Integrating both sides with respect to
$x\in\RR^{d}$ in
(\ref{fismoothappinequ}) and using Fubini's theorem to exchange the order of integration, we obtain
\begin{align}\label{fismoothappinequ1}
\bigl\|\Delta_\theta^{n-1} D^{\alpha}(\phi_\delta \ast \widetilde{f})\bigr\|_{L^{p}(\RR^{d})}^{p}
&\leq \|\theta\|_{1}^{(n-1)p}\int_{[0,1]^{n-1}} \int_{\RR^{d}}\bigl\|D^{n-1} D^{\alpha}(\phi_\delta \ast \widetilde{f})\bigl(x - (\boldsymbol{1}^{\mathsf{T}} t) \theta\bigr) \bigr\|_{1}^{p}dxdt\nonumber\\
&\leq C\|\theta\|_{1}^{(n-1)p}\bigl\|\phi_\delta \ast \widetilde{f}\bigr\|_{W^{n,p}(\RR^{d})}^{p},
\end{align}
where $C=C(d,n,p)>0$ is a  constant.
Therefore, applying (\ref{fismoothappinfi}) and (\ref{fismoothappinequ1}), we obtain the bound
\begin{equation}\label{fismoothappinequ2}
\bigl\| \Delta_\theta^{n-1} D^{\alpha}(\phi_\delta \ast \widetilde{f}) \bigr\|_{L^p(\RR^d)} \le C\|\theta\|_{1}^{n-1} \|\phi_\delta \ast \widetilde{f}\|_{W^{n,p}(\RR^d)},
\end{equation}
for any $1\leq p\leq\infty$ and  any $\alpha\in\NN_{0}^{d}$ with $\|\alpha\|_{1}\leq 1$.
Combining (\ref{smoothappint}) with (\ref{finidffidef}) and (\ref{fismoothappinequ2}), we get
\[
\bigl\|D^{\alpha}(\phi_\delta \ast \widetilde{f} - \widetilde{f}_{\epsilon,\delta}) \bigr\|_{L^p(\RR^{d})}
\le \int_{\RR^{d}} \phi_\epsilon(\theta)
\bigl\| \Delta_\theta^{n-1} D^{\alpha}(\phi_\delta \ast \widetilde{f}) \bigr\|_{L^p(\RR^d)}d\theta
\leq C\epsilon^{n-1} \|\phi_\delta \ast \widetilde{f} \|_{W^{n,p}(\RR^d)},
\]
where the last step uses the facts that $\operatorname{supp} \phi_{\epsilon} = \overline{\BB_{\epsilon}^{d}}$, $\int_{\overline{\BB_{\epsilon}^{d}}}\phi_{\epsilon}(\theta)d\theta=1$,
and that $\|\theta\|_{1}\leq \sqrt{d}\|\theta\|_{2}$.
Hence, by the definition of the $W^{1,p}$-norm, we have
\begin{equation}\label{sobonormapp}
\bigl\|\phi_\delta \ast \widetilde{f} - \widetilde{f}_{\epsilon,\delta}\bigr\|_{W^{1,p}(\RR^{d})}
\leq C\epsilon^{n-1} \bigl\|\phi_\delta \ast \widetilde{f} \bigr\|_{W^{n,p}(\RR^d)},
\end{equation}
for any $1\leq p\leq\infty$.
Since  $\epsilon\in(0,1]$ is fixed, by the limit results (\ref{smoothLpapp1}) and (\ref{smoothLpapp3}),
for $\eta =\epsilon^{n-1}\|f\|_{W^{n,p}(\mathbb{B}^d)}>0$, there exists $\delta_0 = \delta_0(\epsilon,f)>0$ such that for any $\delta$ with $0<\delta<\delta_0$,
\begin{equation}\label{somoonormapp}
\bigl\|\widetilde{f} - \phi_\delta \ast \widetilde{f}\bigr\|_{W^{1,p}(\mathbb{R}^d)} \le \eta = \epsilon^{n-1}\|f\|_{W^{n,p}(\mathbb{B}^d)}.
\end{equation}
Now choose $\delta$ with $0<\delta<\delta_0$. Applying the triangle inequality together with
(\ref{eq:extentheorem}), (\ref{smoothLpapp4}), (\ref{sobonormapp}) and (\ref{somoonormapp}),
 we obtain
\[
\begin{aligned}
\bigl\|f - \widetilde{f}_{\epsilon,\delta}|_{\mathbb{B}^d}\bigr\|_{W^{1,p}(\BB^{d})}
&\le \bigl\|\widetilde{f} - \widetilde{f}_{\epsilon,\delta}\bigr\|_{W^{1,p}(\mathbb{R}^{d})} \\
&\le \bigl\|\widetilde{f} - \phi_\delta \ast \widetilde{f}\bigr\|_{W^{1,p}(\mathbb{R}^{d})}
   + \bigl\|\phi_\delta \ast \widetilde{f} - \widetilde{f}_{\epsilon,\delta}\bigr\|_{W^{1,p}(\mathbb{R}^{d})} \\
&\le \epsilon^{n-1}\|f\|_{W^{n,p}(\mathbb{B}^d)}
   + C\epsilon^{n-1}\|\phi_\delta \ast \widetilde{f}\|_{W^{n,p}(\mathbb{R}^d)} \\
&\le \epsilon^{n-1}\|f\|_{W^{n,p}(\mathbb{B}^d)}+C\epsilon^{n-1}\|\widetilde{f}\|_{W^{n,p}(\mathbb{R}^d)}\\
&\le C\epsilon^{n-1} \|f\|_{W^{n,p}(\BB^d)},
\end{aligned}
\]
where
 $\widetilde{f}_{\epsilon,\delta}|_{\mathbb{B}^d}$ denotes the restriction of $\widetilde{f}_{\epsilon,\delta}$ on $\mathbb{B}^d$
and $C=C(d,n,p)>0$ is a suitable constant. This completes the proof.
\end{proof}

Having established the approximation error in the $W^{1,p}$-norm in Lemma~\ref{Le:SmooFuncAppro}, we now turn to an estimate of the approximating function itself. The following lemma bounds the $W^{s,2}$-norm of $\widetilde{f}_{\epsilon,\delta}|_{\mathbb{B}^d}$ by the $W^{n,p}$-norm of $f$. This estimate will be used later to control the smoothness of the approximant in a higher-order Sobolev space.

\begin{lemma}\label{sobolevL2esti}
Under the same notation and assumptions on $d,n,\phi, f,\widetilde{f},\phi_\delta,\epsilon,\widetilde{f}_{\epsilon,\delta}$ as in
Lemma \ref{Le:SmooFuncAppro}, assume further that $2\le p\le\infty$, $s=(d+3)/2$ and $n\le s$.
Then there exists a constant $C=C(d,n,p)>0$ such that for any $\delta>0$,
\begin{equation}\label{smoothgenacontrol}
\|\widetilde{f}_{\epsilon,\delta}|_{\mathbb{B}^d}\|_{W^{s,2}(\BB^d)} \le C\epsilon^{-(s-n)}\|f\|_{W^{n,p}(\mathbb{B}^d)},
\end{equation}
where
 $\widetilde{f}_{\epsilon,\delta}|_{\mathbb{B}^d}$ denotes the restriction of $\widetilde{f}_{\epsilon,\delta}$ on $\mathbb{B}^d$.
\end{lemma}

\begin{proof}
From the proof of Lemma \ref{Le:SmooFuncAppro},
we have already concluded that for any $\delta>0$, the functions
$\phi_\delta \ast \widetilde{f}$,  $\int_{\mathbb{R}^d} \phi_\epsilon(\theta) (\phi_\delta \ast \widetilde{f})(x-j\theta)d\theta$ and $\widetilde{f}_{\epsilon,\delta}$
all belong to $C_{c}^{\infty}(\RR^{d})$. Moreover,
 for any $n\in\NN_{\geq2}$ with $n\leq s$ and any $2\leq p\leq\infty$, we also have
\begin{equation}\label{smoothL2app4}
\bigl\|\phi_\delta \ast\widetilde{f}\bigr\|_{W^{n,p}(\RR^{d})} \leq \|\widetilde{f}\|_{W^{n,p}(\RR^{d})}.
\end{equation}
Define $\widetilde{f}_{\epsilon,\delta,j}(x)=\int_{\RR^d} \phi_\epsilon(\theta) (\phi_\delta \ast \widetilde{f})(x-j\theta)d\theta$.
Then \begin{equation}\label{approxfuncequ}
\widetilde{f}_{\epsilon,\delta}(x)=\sum_{j=1}^{n-1}\binom{n-1}{j}(-1)^{j-1}\widetilde{f}_{\epsilon,\delta,j}(x),\qquad x\in\mathbb{R}^d,
\end{equation}
and both $\widetilde{f}_{\epsilon,\delta,j}$ and $\widetilde{f}_{\epsilon,\delta}$ belong to $W^{s,2}(\RR^d)$.

Observe that since $n-1$ is fixed, it suffices to bound
$\| \widetilde{f}_{\epsilon,\delta,j}\|_{W^{s,2}(\RR^d)}$
for each fixed integer $j \ge 1$.
To proceed, we change variables. Let $z =j\theta$; then $d\theta = dz / j^d$ and
\[
\begin{aligned}
\widetilde{f}_{\epsilon,\delta,j}(x) &= \int_{\RR^d} \phi_{\epsilon}(\theta) (\phi_\delta \ast \widetilde{f})(x-j\theta)d\theta\\
&= \frac{1}{j^d} \int_{\RR^d} \phi_{\epsilon}\left(\frac{z}{j}\right) (\phi_\delta \ast \widetilde{f})(x-z)dz\\
&= \int_{\RR^d}\phi_{j\epsilon}(z) (\phi_\delta \ast \widetilde{f})(x-z)dz\\
&=(\phi_{j\epsilon} \ast (\phi_\delta \ast \widetilde{f}))(x).
\end{aligned}
\]
Taking the Fourier transform of the convolution representation $\widetilde{f}_{\epsilon,\delta,j} = \phi_{j\epsilon}\ast (\phi_\delta \ast \widetilde{f})$ and applying the convolution theorem, we obtain
\[
\widehat{\widetilde{f}_{\epsilon,\delta,j}}(\xi) = \widehat{\phi_{j\epsilon}}(\xi)\widehat{(\phi_\delta \ast \widetilde{f})}(\xi).
\]
Using the scaling property of the Fourier transform, we have $\widehat{\phi_{j\epsilon}}(\xi) = \widehat{\phi}(j\epsilon\xi)$. Hence, we get
\[
\widehat{\widetilde{f}_{\epsilon,\delta,j}}(\xi) = \widehat{(\phi_\delta \ast \widetilde{f})}(\xi)\widehat{\phi}(j\epsilon\xi).
\]
We now bound the $W^{s,2}(\RR^d)$-seminorm of $\widehat{\widetilde{f}_{\epsilon,\delta,j}}$ as follows:
\begin{equation}\label{soboseminor}
|\widetilde{f}_{\epsilon,\delta,j}|_{W^{s,2}(\RR^d)}^2 \asymp \int_{\RR^d} \|\xi\|_{2}^{2s}\bigl|\widehat{\widetilde{f}_{\epsilon,\delta,j}}(\xi)\bigr|^2 d\xi = \int_{\RR^d} \|\xi\|_{2}^{2s} \bigl|\widehat{(\phi_\delta \ast \widetilde{f})}(\xi)\bigr|^2 |\widehat{\phi}(j\epsilon\xi)|^2 d\xi.
\end{equation}
From the definition of the $W^{n,2}$-norm and the fact that  $\phi_\delta \ast \widetilde{f}$ has compact support, and that for compactly supported functions the Sobolev space $W^{n,p}(\mathbb{R}^d)$ embeds continuously into $W^{n,2}(\mathbb{R}^d)$ when $p \ge 2$, it follows that
\begin{equation}\label{sobonormde}
\int_{\RR^{d}} \|\xi\|_{2}^{2n} \bigl|\widehat{(\phi_\delta \ast \widetilde{f})}(\xi)\bigr|^{2} d\xi\asymp |\phi_\delta \ast \widetilde{f}|_{W^{n,2}(\RR^d)}^{2}\leq \|\phi_\delta \ast \widetilde{f}\|_{W^{n,2}(\RR^d)}^{2}   \le C \|\phi_\delta \ast \widetilde{f}\|_{W^{n,p}(\RR^d)}^{2}.
\end{equation}
Therefore, by applying H\"{o}lder's inequality in (\ref{soboseminor}) together with (\ref{sobonormde}), we get
\begin{align}\label{sobosemisqu}
|\widetilde{f}_{\epsilon,\delta,j} |_{W^{s,2}(\mathbb{R}^d)}^2 &\le \left( \int_{\mathbb{R}^d} \|\xi\|_{2}^{2n} \bigl|\widehat{(\phi_\delta \ast \widetilde{f})}(\xi)\bigr|^2 d\xi \right) \left(\sup_{\xi \in
\mathbb{R}^d} \|\xi\|_{2}^{2(s-n)} |\widehat{\varphi}(j\epsilon\xi)|^{2}\right)\nonumber\\
&\le C \|\phi_\delta \ast \widetilde{f}\|_{W^{n,p}(\mathbb{R}^d)}^2 \left(\sup_{\xi \in \mathbb{R}^d} \|\xi\|_{2}^{2(s-n)} |\widehat{\varphi}(j\epsilon\xi)|^{2}\right).
\end{align}
By the change of variables $\eta = j\epsilon\xi$, we obtain
\begin{equation}\label{Schwartzpro}
\sup_{\xi\in\mathbb{R}^d} \|\xi\|_{2}^{2(s-n)} |\widehat{\phi}(j\epsilon\xi)|^2 = (j\epsilon)^{-2(s-n)} \sup_{\eta\in\mathbb{R}^d} \|\eta\|_{2}^{2(s-n)} |\widehat{\phi}(\eta)|^2 \le C \epsilon^{-2(s-n)},
\end{equation}
since $\sup_{\eta\in\RR^{d}} \|\eta\|_{2}^{2(s-n)} |\widehat{\phi}(\eta)|^2$ is finite
(because $\phi \in C_c^\infty(\mathbb{R}^d)$ implies that its Fourier transform $\widehat{\phi}$ is a Schwartz function).
Therefore, using (\ref{sobosemisqu}) and (\ref{Schwartzpro}), we get
\begin{equation}\label{seminorinque}
|\widetilde{f}_{\epsilon,\delta,j}|_{W^{s,2}(\mathbb{R}^d)} \le C \epsilon^{-(s-n)}\|\phi_\delta \ast \widetilde{f}\|_{W^{n,p}(\mathbb{R}^d)}.
\end{equation}
Furthermore,  we have
\begin{align}\label{squnorinque}
\|\widetilde{f}_{\epsilon,\delta,j}\|_{L^{2}(\mathbb{R}^d)} &=\left \|\int_{\RR^d} \phi_{\epsilon}(\theta) (\phi_\delta \ast \widetilde{f})(x-j\theta)d\theta\right\|_{L^{2}(\mathbb{R}^d)}\nonumber\\
&\leq \int_{\RR^d} \phi_{\epsilon}(\theta)\|(\phi_\delta \ast \widetilde{f})(x-j\theta)\|_{L^{2}(\mathbb{R}^d;dx)}d\theta\nonumber\\
&=\|\phi_\delta \ast \widetilde{f}\|_{L^{2}(\mathbb{R}^d)}\leq \|\phi_\delta \ast \widetilde{f}\|_{W^{n,2}(\mathbb{R}^d)}\leq C\|\phi_\delta \ast \widetilde{f}\|_{W^{n,p}(\mathbb{R}^d)},
\end{align}
where the second step uses  Minkowski's integral inequality, and the third step  uses $\int_{\RR^{d}} \phi_{\epsilon}(\theta)d\theta=1$
and the translation invariance of the $L^{2}$-norm.
By the definition of the $W^{s,2}$-norm, we have
\begin{align}\label{normeque}
\| \widetilde{f}_{\epsilon,\delta,j}\|_{W^{s,2}(\RR^d)}^2 &\asymp \int_{\RR^d} (1+\|\xi\|_{2})^{2s} \bigl|\widehat{\widetilde{f}_{\epsilon,\delta,j}}(\xi)\bigr|^2 d\xi \nonumber\\
&\asymp  \int_{\RR^d}\bigl|\widehat{\widetilde{f}_{\epsilon,\delta,j}}(\xi)\bigr|^2 d\xi+\int_{\RR^d}  \|\xi\|_{2}^{2s}\bigl|\widehat{\widetilde{f}_{\epsilon,\delta,j}}(\xi)\bigr|^2 d\xi\nonumber\\
&\asymp \|\widetilde{f}_{\epsilon,\delta,j}\|_{L^{2}(\mathbb{R}^d)}^{2}+|\widetilde{f}_{\epsilon,\delta,j}|_{W^{s,2}(\mathbb{R}^d)}^{2},
\end{align}
where the second step uses the fact that $(1+\|\xi\|_{2})^{2s} \asymp 1 + \|\xi\|_{2}^{2s}$ for all $\xi \in \mathbb{R}^d$.
Consequently, combining (\ref{seminorinque}) with (\ref{squnorinque}) and (\ref{normeque}), we obtain
\begin{align}\label{normineque1}
\| \widetilde{f}_{\epsilon,\delta,j}\|_{W^{s,2}(\RR^d)}
&\leq C\left(\|\widetilde{f}_{\epsilon,\delta,j}\|_{L^{2}(\mathbb{R}^d)}^{2}+|\widetilde{f}_{\epsilon,\delta,j}|_{W^{s,2}(\mathbb{R}^d)}^{2}\right)^{1/2}\nonumber\\
&\leq C\left(1+\epsilon^{-2(s-n)}\right)^{1/2}\|\phi_\delta \ast \widetilde{f}\|_{W^{n,p}(\mathbb{R}^d)}\nonumber\\
&\leq C\epsilon^{-(s-n)}\|\phi_\delta \ast \widetilde{f}\|_{W^{n,p}(\mathbb{R}^d)},
\end{align}
where the last step uses that $\epsilon \in (0,1]$ is fixed.
Finally, by the triangle inequality together with (\ref{eq:extentheorem}), (\ref{smoothL2app4}), (\ref{approxfuncequ}) and (\ref{normineque1}),
we obtain the following estimate for the $W^{s,2}(\BB^d)$-norm of  the restriction function $\widetilde{f}_{\epsilon,\delta}|_{\mathbb{B}^d}$:
\[
\begin{aligned}
\|\widetilde{f}_{\epsilon,\delta}|_{\mathbb{B}^d}\|_{W^{s,2}(\BB^d)}
&\leq\|\widetilde{f}_{\epsilon,\delta}\|_{W^{s,2}(\RR^d)}\\
&\leq \sum_{j=1}^{n-1}\binom{n-1}{j}\|\widetilde{f}_{\epsilon,\delta,j}\|_{W^{s,2}(\mathbb{R}^d)}\\
&\leq C\epsilon^{-(s-n)}\|\phi_\delta \ast \widetilde{f}\|_{W^{n,p}(\mathbb{R}^d)}\\
&\leq C\epsilon^{-(s-n)}\|\widetilde{f}\|_{W^{n,p}(\mathbb{R}^d)}\\
&\leq C\epsilon^{-(s-n)}\|f\|_{W^{n,p}(\mathbb{B}^d)},
\end{aligned}
\]
where $C=C(d,n,p)>0$ is a suitable constant.
This completes the proof.
\end{proof}

Now, we briefly recall the definition of variation spaces corresponding to  ReLU neural networks.
For details and the intuition behind this definition, we refer the reader to \citep{siegel2024sharp,siegel2023characterization}.
The unit ball of the variation space, also known as the Barron space, is the closed symmetric convex hull of the dictionary $\PP_1^d$ defined as in (\ref{definenet5}), i.e.,
\begin{equation}\label{eq:convexhullsobolev}
B_1(\PP_1^d) := \overline{\left\{ \sum_{i=1}^W a_i \eta_i : a_i \in \mathbb{R}, \eta_i\in\PP_1^d,\ \sum_{i=1}^W |a_i|\le 1 \right\}},
\end{equation}
where the closure is taken in the $W^{1,\infty}$-norm.
The variation space is defined as
\begin{equation}\label{variationspace}
\mathcal{K}_1(\PP_1^d) := \{ f\in W^{1,\infty}(\mathbb{B}^d) : \|f\|_{\mathcal{K}_1(\PP_1^d)} < \infty \},
\end{equation}
where the norm is given by
\begin{equation}\label{variationspace1}
\|f\|_{\mathcal{K}_1(\PP_1^d)} = \inf\{ c > 0 : f \in c B_1(\PP_1^d) \}.
\end{equation}
 It is known that the closure in (\ref{eq:convexhullsobolev}) is the same when taken in different norms, such as the weaker $L^2$-norm (see \citep{mao2026approximation,siegel2025optimal,yang2024nonparametric})
and thus the variation space remains the same.
We remark that the variation space can be defined for a general dictionary $\DD$, that is, a bounded set of functions
(see, e.g., \citep{DeVore1998nonlinear,mhaskar2004tractability,mhaskar2024tractability,siegel2024sharp}).
The variation space is of great importance in nonlinear dictionary approximation and in the convergence theory of greedy algorithms
\citep{deVore1996advances,siegel2023optimal,temlyakov2011greedy,temlyakov2008greedy}.
Moreover, the variation spaces (the Barron spaces)
play a key  role in the theory of shallow neural networks, and they have been extensively studied in various forms  in recent years \citep{bach2017breaking,weinan2019priori,weinan2022barron,parhi2021banach,parhi2022what,
siegel2023characterization,siegel2024sharp,siegel2025optimal,yang2025optimal,mao2026approximation}.

The next lemma follows from \citep[Theorem 3] {siegel2025optimal}, which showed that the (nearly)
optimal  approximation rate for shallow ReLU
neural networks on the variation space $\mathcal{K}_1(\PP_1^{d})$  holds up to logarithmic factors.

\begin{lemma}\label{varisobolevesti}
Let $d\in\mathbb{N}$. For any $f\in\mathcal{K}_1(\PP_1^d)$,  there exists a neural network $f_W\in\Sigma_W(\PP_1^d)$ with the weights $a_i$ satisfying
\[
\sum_{i=1}^W |a_i| \le \|f\|_{\mathcal{K}_1(\PP_1^d)}
\]
such that
\[
\| f - f_W \|_{W^{1,\infty}(\mathbb{B}^d)} \le C \| f \|_{\mathcal{K}_1(\PP_1^d)} W^{-\frac12-\frac{1}{2d}},
\]
where $C=C(d)>0$ is a constant.
\end{lemma}

The following  lemma follows from \citep[Theorem 1.1]{mao2026approximation}, which showed
 that the $L^{2}$-Sobolev space with a certain amount of smoothness embeds into the variation space
$\mathcal{K}_1(\PP_1^{d})$ using a Radon space
characterization of the variation space. This embedding is sharp in the sense of metric entropy and plays a key role in our research.

\begin{lemma}\label{Sobolevembed}
Let $d\in\NN$ and $s = (d + 3)/2$. Then we have the embedding
\[
W^{s,2}(\BB^{d}) \subset \mathcal{K}_1(\PP^{d}_{1}).
\]
\end{lemma}

By Lemma \ref{varisobolevesti} and Lemma \ref{Sobolevembed}, we obtain the following approximation result.
\begin{lemma}\label{Soboapplinear}
Let $d\in\NN$,$s = (d + 3)/2$ and $1 \le p \le \infty$. For any $f\in C_{c}^{\infty}(\mathbb{R}^d)$,
there exists a neural network $f_W\in\Sigma_W(\PP_1^d)$ with the weights $a_i$ satisfying
\[
\sum_{i=1}^W |a_i| \le  C \| f|_{\mathbb{B}^d} \|_{W^{s,2}(\BB^{d})}
\]
such that
\[
\| f|_{\mathbb{B}^d}- f_W \|_{W^{1,p}(\BB^{d})} \le C \| f|_{\mathbb{B}^d} \|_{W^{s,2}(\BB^{d})} W^{-\frac{s-1}{d}},
\]
where $f|_{\mathbb{B}^d}$ denotes the restriction of $f$ on $\mathbb{B}^d$, and $C = C(d,p)>0$ is a constant.
\end{lemma}
\begin{proof}
For any $1 \le p \le \infty$, we know that the Sobolev space $W^{1,\infty}(\mathbb{B}^d)$ embeds continuously into $W^{1,p}(\mathbb{B}^d)$.
Since $f\in C_{c}^{\infty}(\mathbb{R}^d)$, its restriction $f|_{\mathbb{B}^d}$ belongs to $C^\infty(\mathbb{B}^d)$,
and hence to $W^{1,p}(\mathbb{B}^d)$ and $W^{s,2}(\mathbb{B}^d)$.
Furthermore, together with
Lemma \ref{Sobolevembed},
this implies that  $f|_{\mathbb{B}^d}\in\mathcal{K}_1(\PP_1^d)$.
Therefore,  by Lemma \ref{varisobolevesti} and Lemma \ref{Sobolevembed}, there exists a network $f_W\in\Sigma_W(\PP_1^d)$ with the weights $a_i$ satisfying
\[
\sum_{i=1}^W |a_i| \leq\| f|_{\mathbb{B}^d} \|_{\mathcal{K}_1(\PP_1^{d})}\le  C \| f|_{\mathbb{B}^d} \|_{W^{s,2}(\BB^{d})}
\]
such that
\[
\begin{aligned}
 \| f|_{\mathbb{B}^d}- f_W \|_{W^{1,p}(\BB^{d})}
&\leq C \| f|_{\mathbb{B}^d} - f_W \|_{W^{1,\infty}(\BB^{d})} \\
&\leq C \| f|_{\mathbb{B}^d} \|_{\mathcal{K}_1(\PP_1^{d})} W^{-\frac{s-1}{d}} \\
&\leq C \| f|_{\mathbb{B}^d} \|_{W^{s,2}(\BB^{d})} W^{-\frac{s-1}{d}},
\end{aligned}
\]
where
$C = C(d,p)>0$ is a suitable constant.
\end{proof}

 We are now ready to prove Theorem \ref{Soboappmainresult}.

\begin{proof}[Proof of  Theorem \ref{Soboappmainresult}]
We divide the proof into two steps. First, we construct a smooth approximant $\widetilde{f}_{\epsilon,\delta}$ of $f$ via localized convolution sums. Second, we approximate $\widetilde{f}_{\epsilon,\delta}$ by a shallow norm  constrained neural network.

For the first step, we fix $\epsilon \in (0,1]$ to be chosen later and construct $\widetilde{f}_{\epsilon,\delta}$ as in (\ref{approxfunc}) from Lemma \ref{Le:SmooFuncAppro}. The function $\widetilde{f}_{\epsilon,\delta}$ belongs to $C_c^\infty(\mathbb{R}^d)$; its smoothness and compact support were established in the proof of Lemma \ref{Le:SmooFuncAppro}.
 By Lemma \ref{Le:SmooFuncAppro}, there exists a sufficiently small $\delta > 0$ (depending on $\epsilon$ and $f$) such that
\begin{equation}\label{smoothsoboapp1}
\|f - \widetilde{f}_{\epsilon,\delta}|_{\mathbb{B}^d}\|_{W^{1,p}(\BB^d)} \le C\|f\|_{W^{n,p}(\BB^d)}\epsilon^{n-1},
\end{equation}
where $\widetilde{f}_{\epsilon,\delta}|_{\mathbb{B}^d}$ denotes the restriction of $\widetilde{f}_{\epsilon,\delta}$ on $\mathbb{B}^d$.

In the second step,
by  Lemma \ref{sobolevL2esti} and Lemma \ref{Soboapplinear},
there exists a network $f_W\in\Sigma_W(\PP_1^d)$ with the weights $a_i$ satisfying
\[
\sum_{i=1}^W |a_i| \le  C \| \widetilde{f}_{\epsilon,\delta}|_{\mathbb{B}^d} \|_{W^{s,2}(\BB^{d})}\le C\|f\|_{W^{n,p}(\BB^d)}\epsilon^{-(s-n)}
\]
such that
\begin{align}\label{smoothsoboapp2}
&\| \widetilde{f}_{\epsilon,\delta}|_{\mathbb{B}^d} - f_W \|_{W^{1,p}(\mathbb{B}^d)}\nonumber\\
&\quad \le C \| \widetilde{f}_{\epsilon,\delta}|_{\mathbb{B}^d}\|_{W^{s,2}(\mathbb{B}^d)} W^{-\frac{s-1}{d}}\nonumber\\
&\quad\le C\|f\|_{W^{n,p}(\BB^d)}\epsilon^{-(s-n)} W^{-\frac{s-1}{d}}.
\end{align}
Applying the triangle inequality together with (\ref{smoothsoboapp1}) and (\ref{smoothsoboapp2}), for any $f\in \mathcal{F}_{d,n,p}$, we get
\begin{align}\label{smoothsoboapp3}
&\| f - f_W \|_{W^{1,p}(\BB^{d})}\nonumber\\
&\quad\leq \|f-  \widetilde{f}_{\epsilon,\delta}|_{\mathbb{B}^d}  \|_{W^{1,p}(\BB^{d})}
+\|  \widetilde{f}_{\epsilon,\delta}|_{\mathbb{B}^d} - f_W \|_{W^{1,p}(\BB^{d})}\nonumber\\
&\quad \le C\epsilon^{n-1}
+ C\epsilon^{-(s-n)}W^{-\frac{s-1}{d}}.
\end{align}
Now we consider two cases depending on the size of $K$ relative to $W$.

{\bf Case 1:} $K \ge C W^{\frac{s-n}{d}}$. We choose $\epsilon = W^{-1/d}$ in (\ref{smoothsoboapp3}). Then
\[
\| f - f_W \|_{W^{1,p}(\mathbb{B}^d)} \le C W^{-\frac{n-1}{d}},
\]
and the weights satisfy
\[
\sum_{i=1}^W |a_i| \le C\epsilon^{-(s-n)} = C W^{\frac{s-n}{d}} \le K,
\]
so $f_W\in\Sigma_W^{K}(\PP_1^d)$.

{\bf Case 2:} $K \le C W^{\frac{s-n}{d}}$. Then we have $W^{-1} \le C K^{-\frac{d}{s-n}}$ and $W^{-\frac{s-1}{d}} \le C K^{-\frac{s-1}{s-n}}$. Choose $\epsilon$ such that $C\epsilon^{-(s-n)} = K$, i.e., $\epsilon = (C/K)^{1/(s-n)}$. Substituting this into (\ref{smoothsoboapp3}) gives
\[
\| f - f_W \|_{W^{1,p}(\mathbb{B}^d)} \le C K^{-\frac{n-1}{s-n}},
\]
and the weights satisfy $\sum_{i=1}^W |a_i| \le C\epsilon^{-(s-n)} = K$, i.e., $f_W\in\Sigma_W^{K}(\PP_1^d)$.
Combining the two cases yields the desired result, where $C = C(d,n,p)>0$ is a constant. This completes the proof.
\end{proof}

\section{Proof of Theorem \ref{upperboundapp}}\label{ApproximadeepReLU}
This section is devoted to the proof of Theorem \ref{upperboundapp}.
The main idea is based on the common strategy of approximating
$f$ by localized polynomials, which in turn are approximated by deep norm  constrained neural networks.
Following the constructions in \citep{jiao2023approximation, guhring2020error,guhring2021approximation,
lu2021deep, hon2022simultaneous,jiao2025drm,jiao2024convergence},
we first consider the approximation
of the  function $x\mapsto x^{2}$ and  then extend the approximation to monomials.
Finally, using a localized Taylor expansion and partitions of unity, we obtain the desired approximation for $f$.

\begin{lemma}\label{quadNNapp}
For any $k\in\NN,$ there exists a neural network $\phi_{k}\in\mathcal{NN}(k,1,3)$ such that
\[
\|x^{2}-\phi_{k}(x)\|_{W^{1,\infty}([0,1])}\leq\frac{1}{k}.
\]
\end{lemma}

\begin{proof}
It is shown in the proof of \citep[Lemma 3.3]{jiao2023approximation} that for any $k\in\NN$ there exists
$\phi_{k}\in\mathcal{NN}(k,1,3)$ with
\begin{equation}\label{Def:quadNN}
\phi_{k}(x)=\frac{1}{k}\sum_{i=1}^{k}2\sigma\left(x-\frac{2i-1}{2k}\right)
\end{equation}
such that
\[
\|x^{2}-\phi_{k}(x)\|_{L^{\infty}([0,1])}\leq\frac{1}{2k^{2}}.
\]
For $k=1$, by (\ref{Def:quadNN}), the derivative of $\phi_k$ is given by
\[
\phi_k'(x)=
\begin{cases}
0, & x\in [0,\frac{1}{2}),\\
2, & x\in(\frac{1}{2},1].
\end{cases}
\]
For $k\ge 2$, the derivative of $\phi_k$ is given by
\[
\phi_k'(x)=
\begin{cases}
0, & x\in[0,\tfrac{1}{2k}),\\
\frac{2i}{k}, & x\in\left(\frac{2i-1}{2k},\frac{2i+1}{2k}\right), i=1,\dots,k-1,\\
2, & x\in\left(\frac{2k-1}{2k},1\right].
\end{cases}
\]
Therefore,  for $k=1$, we get
\[
|x^{2}-\phi_{k}(x)|_{W^{1,\infty}([0,\frac{1}{2}))}=\|2x-\phi_k'(x)\|_{L^{\infty}([0,\frac{1}{2}))}
=\|2x\|_{L^{\infty}([0,\frac{1}{2}))}\leq 1
\]
and
\[
|x^{2}-\phi_{k}(x)|_{W^{1,\infty}((\frac{1}{2},1])}=\|2x-\phi_k'(x)\|_{L^{\infty}((\frac{1}{2},1])}
=\|2x-2\|_{L^{\infty}((\frac{1}{2},1])}\leq 1.
\]
When $k\geq 2$, we get
\[
|x^{2}-\phi_{k}(x)|_{W^{1,\infty}([0,\frac{1}{2k}))}=\|2x-\phi_k'(x)\|_{L^{\infty}([0,\frac{1}{2k}))}
=\|2x\|_{L^{\infty}([0,\frac{1}{2k}))}\leq \frac{1}{k}
\]
and for each $i=1,\ldots, k-1$,
\[
\begin{aligned}
&|x^{2}-\phi_{k}(x)|_{W^{1,\infty}\left(\left(\frac{2i-1}{2k},\frac{2i+1}{2k}\right)\right)}\\
&=\left\|2x-\phi_k'(x)\right\|_{L^{\infty}\left(\left(\frac{2i-1}{2k},\frac{2i+1}{2k}\right)\right)}\\
&=\left\|2x-\frac{2i}{k}\right\|_{L^{\infty}\left(\left(\frac{2i-1}{2k},\frac{2i+1}{2k}\right)\right)}\\
&\leq\max\left\{\left|2\cdot \frac{2i-1}{2k}-\frac{2i}{k}\right|, \left|2\cdot\frac{2i+1}{2k}-\frac{2i}{k}\right|\right\}\\
&=\frac{1}{k}
\end{aligned}
\]
and
\[
\begin{aligned}
|x^{2}-\phi_{k}(x)|_{W^{1,\infty}\left(\left(\frac{2k-1}{2k},1\right]\right)}&=\left\|2x-\phi_k'(x)\right\|_{L^{\infty}
\left(\left(\frac{2k-1}{2k},1\right]\right)}\\
&=\left\|2x-2\right\|_{L^{\infty}\left(\left(\frac{2k-1}{2k},1\right]\right)}\leq 2-\frac{2k-1}{k}=\frac{1}{k}.
\end{aligned}
\]
Combining these estimates yields
\[
|x^{2}-\phi_{k}(x)|_{W^{1,\infty}([0,1])}\leq \frac{1}{k}.
\]
Finally, for any  $k\in\NN,$  we get
\[
\begin{aligned}
\|x^{2}-\phi_{k}(x)\|_{W^{1,\infty}([0,1])}&=\max\left\{\|x^{2}-\phi_{k}(x)\|_{L^{\infty}([0,1])},|x^{2}-\phi_{k}(x)|_{W^{1,\infty}([0,1])} \right\}\\
&\leq\max\left\{\frac{1}{2k^{2}}, \frac{1}{k}\right\}
= \frac{1}{k}.
\end{aligned}
\]
This completes the proof.
\end{proof}

Using the polarization identity
\[
xy=2\left(\left(\frac{x+y}{2}\right)^{2}-\left(\frac{x}{2}\right)^{2}-\left(\frac{y}{2}\right)^{2}\right)
\quad\text{for } x,y\in\RR,
\]
we can approximate the product function by neural networks.

\begin{lemma}\label{productNNapp}
For any $k\in\NN,$ there exists a neural network $\psi_{k}\in\mathcal{NN}(3k,1,18)$ such that
\[
\|xy-\psi_{k}(x,y)\|_{W^{1,\infty}([0,1]^{2})}\leq\frac{6}{k}.
\]
\end{lemma}
\begin{proof}	
By Lemma \ref{quadNNapp}, there exists a network $\phi_{k}\in\mathcal{NN}(k,1,3)$ such that
\[
\|x^{2}-\phi_{k}(x)\|_{W^{1,\infty}([0,1])}\leq\frac{1}{k}.
\]
Using the fact that
\[
xy=2\left(\left(\frac{x+y}{2}\right)^{2}-\left(\frac{x}{2}\right)^{2}-\left(\frac{y}{2}\right)^{2}\right),
\]
we consider the function
\[
\psi_{k}(x,y)=2\phi_{k}\left(\frac{1}{2}x+\frac{1}{2}y\right)-2\phi_{k}\left(\frac{1}{2}x\right)
-2\phi_{k}\left(\frac{1}{2}y\right).
\]
By the triangle inequality together with Lemma \ref{quadNNapp}, for any $x,y\in [0,1]$, we get
\[
\begin{aligned}
\|xy-\psi_{k}(x,y)\|_{W^{1,\infty}([0,1]^{2})}
&\leq 2\left\|\left(\frac{x+y}{2}\right)^{2}-\phi_{k}\left(\frac{x+y}{2}\right)\right\|_{W^{1,\infty}([0,1]^{2})}\\
&\quad+2\left\|\left(\frac{x}{2}\right)^{2}-\phi_{k}\left(\frac{x}{2}\right)\right\|_{W^{1,\infty}([0,1]^{2})}\\
&\quad+2\left\|\left(\frac{y}{2}\right)^{2}-\phi_{k}\left(\frac{y}{2}\right)\right\|_{W^{1,\infty}([0,1]^{2})}\\
&\leq \frac{6}{k}.
\end{aligned}
\]
By Lemma \ref{NN_Com-con}, $\psi_{k}\in\mathcal{NN}(3k,1,18)$. This completes the proof.

\end{proof}

By rescaling, we have the following modification of Lemma \ref{productNNapp}.

\begin{lemma}\label{RescaproductNNapp}
For any $k\in\NN$ and $a,b\in\RR$ with $a<b$, there exists a neural network
\[\psi_{k}\in \mathcal{NN}\left(3k+2, 1, 18(b-a)^{2}\cdot\max\left\{\frac{1+|a|}{b-a}, 1\right\}+4|a|+2a^{2}\right)\]
 such that
\[
\|xy-\psi_{k}(x,y)\|_{W^{1,\infty}([a,b]^{2})}
\leq (b-a)^{2}\max\left\{1, \frac{1}{b-a}\right\}\cdot\frac{6}{k}.
\]
\end{lemma}
\begin{proof}	
By Lemma \ref{productNNapp},  there exists $\widetilde{\psi}_{k}\in\mathcal{NN}(3k,1,18)$ such that
\[
\left\|\widetilde{x}\widetilde{y}-\widetilde{\psi}_{k}(\widetilde{x},\widetilde{y})\right\|_{W^{1,\infty}([0,1]^{2})}\leq\frac{6}{k}.
\]
Setting $\widetilde{x}=\frac{x-a}{b-a}$ and $\widetilde{y}=\frac{y-a}{b-a}$ for any $x,y\in [a,b]$, we have
$\widetilde{x}, \widetilde{y}\in [0,1]$.
By a change of variables and the chain rule, we obtain the estimate
\[
\left\|\widetilde{\psi}_{k}\left(\frac{x-a}{b-a},\frac{y-a}{b-a}\right)-\left(\frac{x-a}{b-a}\cdot\frac{y-a}{b-a}\right)\right\|_{W^{1,\infty}([a,b]^{2})}
\leq \max\left\{1, \frac{1}{b-a}\right\}\cdot\frac{6}{k}.
\]
It follows that
\[
\left\|(b-a)^{2}\widetilde{\psi}_{k}\left(\frac{x-a}{b-a},\frac{y-a}{b-a}\right)+a(x+y)-a^{2}-xy\right\|_{W^{1,\infty}([a,b]^{2})}
\leq (b-a)^{2}\max\left\{1, \frac{1}{b-a}\right\}\cdot\frac{6}{k}.
\]
We define the following network for any $x,y\in [a,b]$,
\[
\psi_{k}(x,y)=(b-a)^{2}\widetilde{\psi}_{k}\left(\frac{x-a}{b-a},\frac{y-a}{b-a}\right)+\sigma(ax+ay-a^{2})
-\sigma(-ax-ay+a^{2}).
\]
Using the fact that $\sigma(t)-\sigma(-t)=t$ for all $t\in\mathbb{R}$, it follows that $\sigma(ax+ay-a^{2})-\sigma(-ax-ay+a^{2})=ax+ay-a^{2}$ for any $x,y\in[a,b]$. Consequently,
\[
\psi_{k}(x,y)=(b-a)^{2}\widetilde{\psi}_{k}\left(\frac{x-a}{b-a},\frac{y-a}{b-a}\right)+ax+ay-a^{2}.
\]
Therefore, we have
\[
\begin{aligned}
&\|\psi_{k}(x,y)-xy\|_{W^{1,\infty}([a,b]^{2})}\\
&=\left\|(b-a)^{2}\widetilde{\psi}_{k}\left(\frac{x-a}{b-a},\frac{y-a}{b-a}\right)+a(x+y)-a^{2}-xy\right\|_{W^{1,\infty}([a,b]^{2})}\\
&\leq (b-a)^{2}\max\left\{1, \frac{1}{b-a}\right\}\cdot\frac{6}{k}.
\end{aligned}
\]
Furthermore, by Lemma \ref{NN_Com-con}, we get
\[\widetilde{\psi}_{k}\left(\frac{x-a}{b-a},\frac{y-a}{b-a}\right)\in \mathcal{NN}
\left(3k, 1, 18\cdot\max\left\{\frac{1+|a|}{b-a}, 1\right\}\right)\]
 and
\[
\sigma(ax+ay-a^{2}), \sigma(-ax-ay+a^{2})\in \mathcal{NN}(1,1,2|a|+a^{2}),
\]
and finally
\[
\psi_{k}(x,y)\in \mathcal{NN}\left(3k+2, 1, 18(b-a)^{2}\cdot\max\left\{\frac{1+|a|}{b-a}, 1\right\}+4|a|+2a^{2}\right).
\]
This completes the proof.
\end{proof}

To approximate general functions in Sobolev spaces, we first need to approximate monomials efficiently. We can then obtain the following lemma through induction.

\begin{lemma}\label{mulproductNNapp}
For any $d\in\NN_{\geq 2}$ and $k\in\NN$ with $k\geq 4^{\lceil \log_{2}d\rceil +1}$, there exists a neural network
$\phi\in  \mathcal{NN}(d(3k+2), \lceil \log_{2}d\rceil, 27^{\lceil \log_{2}d\rceil})$
such that
\[
\left\|x_{1}\cdots x_{d}-\phi(x)\right\|_{W^{1,\infty}\left([0,1]^{d}\right)}
\leq 6d^{2} k^{-1},\quad x=(x_{1},\ldots, x_{d})^{\mathsf{T}}\in [0,1]^{d}.
\]
\end{lemma}
\begin{proof}
We first consider the case $d=2^m$ for some $m\in\mathbb{N}$. Let $\phi_1$ be the neural network from Lemma \ref{RescaproductNNapp}. For any $x_i\in[0,1]$, $i=1,\dots,2^m$, we define $\phi_m$ via recursive construction by
\begin{equation}\label{induNN}
\phi_m(x_1,\dots,x_{2^m}) = \phi_1\bigl(\phi_{m-1}(x_1,\dots,x_{2^{m-1}}),\ \phi_{m-1}(x_{2^{m-1}+1},\dots,x_{2^m})\bigr),
\end{equation}
where $\phi_0(t)=t$ for $t\in[0,1]$.

Now, we inductively show that
$\phi_{m}$ satisfies
\begin{equation}\label{inductproduct}
\left\|x_{1}\cdots x_{2^{m}}-\phi_{m}(x_{1},\ldots, x_{2^{m}})\right\|_{W^{1,\infty}\left([0,1]^{2^{m}}\right)}
\leq 4^{m-1}\cdot 6k^{-1}.
\end{equation}
It is obvious that the assertion is true for $m=1$ by construction and Lemma \ref{RescaproductNNapp}. Assume that the assertion
is true for some $m-1\in\NN$. We will prove that it is true for $m$.
By the induction hypothesis (\ref{inductproduct}),
in particular, we have
\[
|x_1\cdots x_{2^{m-1}} - \phi_{m-1}(x_{1},\ldots, x_{2^{m-1}})| \le 4^{m-2}\cdot 6k^{-1},
\]
and for each $h=1,\dots,2^{m-1}$,
\[
\left|\prod_{i=1, i\ne h}^{2^{m-1}}x_i - \frac{\partial \phi_{m-1}(x_{1},\ldots, x_{2^{m-1}})}{\partial x_h}\right| \le 4^{m-2}\cdot 6k^{-1}.
\]
Since $0\le x_1\cdots x_{2^{m-1}}\le 1$ and
$0\le \prod_{i=1, i\ne h}^{2^{m-1}} x_i\leq 1$,
it follows that
\begin{equation}\label{inductproduresu1}
\phi_{m-1}(x_{1},\ldots, x_{2^{m-1}})
\in [-4^{m-2}\cdot 6k^{-1},\, 1+4^{m-2}\cdot 6k^{-1}]
\end{equation}
and
\begin{equation}\label{inductproduresu2}
\frac{\partial \phi_{m-1}(x_{1},\ldots, x_{2^{m-1}})}{\partial x_{h}}
\in [-4^{m-2}\cdot 6k^{-1}, 1+4^{m-2}\cdot 6k^{-1}], \quad h=1,\dots,2^{m-1}.
\end{equation}
Similarly, applying the same estimates from the induction hypothesis (\ref{inductproduct}) to the variables $x_{2^{m-1}+1},\dots,x_{2^{m}}$ yields,
\begin{equation}\label{inductproduresu11}
\phi_{m-1}(x_{2^{m-1}+1},\ldots, x_{2^{m}})
\in [-4^{m-2}\cdot 6k^{-1},\, 1+4^{m-2}\cdot 6k^{-1}].
\end{equation}
By  the construction of $\phi_{m}$, we have
\[
\begin{aligned}
&\left\|x_{1}\cdots x_{2^{m}}-\phi_{m}(x_{1},\ldots, x_{2^{m}})\right\|_{W^{1,\infty}\left([0,1]^{2^{m}}\right)}\\
&=\left\|x_{1}\cdots x_{2^{m-1}}\cdot x_{2^{m-1}+1}\cdots x_{2^{m}}\right.\\
&\qquad\qquad\qquad\left.-\phi_{1}(\phi_{m-1}(x_{1},\ldots, x_{2^{m-1}}), \phi_{m-1}(x_{2^{m-1}+1},\ldots, x_{2^{m}}))\right\|_{W^{1,\infty}\left([0,1]^{2^{m}}\right)}\\
&\leq\left\|x_{1}\cdots x_{2^{m-1}}\cdot (x_{2^{m-1}+1}\cdots x_{2^{m}})-\phi_{m-1}(x_{1},\ldots, x_{2^{m-1}})\cdot (x_{2^{m-1}+1}\cdots x_{2^{m}})\right\|_{W^{1,\infty}\left([0,1]^{2^{m}}\right)}\\
&\qquad+\left\|\phi_{m-1}(x_{1},\ldots, x_{2^{m-1}})\cdot (x_{2^{m-1}+1}\cdots x_{2^{m}})\right.\\
&\qquad\qquad\qquad\left.-\phi_{m-1}(x_{1},\ldots, x_{2^{m-1}})\cdot \phi_{m-1}(x_{2^{m-1}+1},\cdots, x_{2^{m}})\right\|_{W^{1,\infty}\left([0,1]^{2^{m}}\right)}\\
&\qquad+\left\|\phi_{m-1}(x_{1},\ldots, x_{2^{m-1}})\cdot \phi_{m-1}(x_{2^{m-1}+1},\ldots, x_{2^{m}})\right.\\
&\qquad\qquad\qquad\left.-\phi_{1}(\phi_{m-1}(x_{1},\ldots, x_{2^{m-1}}), \phi_{m-1}(x_{2^{m-1}+1},\ldots, x_{2^{m}}))\right\|_{W^{1,\infty}\left([0,1]^{2^{m}}\right)}\\
& \leq 4^{m-2}\cdot 6 k^{-1}+(1+4^{m-2}\cdot 6 k^{-1})\cdot 4^{m-2}\cdot 6 k^{-1}
+(1+2\cdot 4^{m-2}\cdot 6 k^{-1})^{2}\cdot 6 k^{-1}\\
&\leq  4^{m-1}\cdot 6 k^{-1}.
\end{aligned}
\]
where the second step follows from the triangle inequality, the third step uses the induction hypothesis (\ref{inductproduct})
and Lemma \ref{RescaproductNNapp} together with  (\ref{inductproduresu1}), (\ref{inductproduresu2}) and (\ref{inductproduresu11}), and the last step follows by elementary algebra when $k\geq 4^{\lceil \log_{2}d\rceil +1}= 4^{m+1}$, since then $4^{m-2}\cdot 6k^{-1}\le 6/64< 0.1$.
Hence, the assertion is true for $m$.

Since $k\geq  4^{m+1}$, for any $1\le s\le m-1$, we have
\[
4^{s-1}\cdot 6k^{-1} \le 4^{m-2}\cdot 6k^{-1} \le \frac{6}{64} < 0.1.
\]
For recursive construction of $\phi_m$ in (\ref{induNN}),
the outputs of each approximation stage $\phi_{s}$ become the inputs for the next stage.
By the induction hypothesis (which holds for every  recursion level $s\le m-1$), it follows that for all
$x_i\in[0,1]$, $i=1,\ldots, 2^{s+1}$,
\[
\phi_s(x_1,\dots,x_{2^s})\in [-4^{s-1}\cdot 6 k^{-1}, 1+4^{s-1}\cdot 6 k^{-1}]\subseteq [-0.1, 1.1]
\]
and
\[
\phi_s(x_{2^s+1},\dots,x_{2^{s+1}}) \in [-4^{s-1}\cdot 6 k^{-1}, 1+4^{s-1}\cdot 6 k^{-1}]\subseteq [-0.1, 1.1].
\]
Let $a=-0.1$ and $b=1.1$ in Lemma \ref{RescaproductNNapp}.
Thus, by Lemma \ref{RescaproductNNapp},
for any $x_1,x_2\in[-0.1,1.1]$, we have
\[
\phi_1(x_1,x_2)\in \mathcal{NN}(3k+2, 1, 27).
\]
Finally, by the recursive construction of $\phi_m$ and repeated application of  composition and concatenation in Lemma \ref{NN_Com-con}, we obtain
\[
\phi_m \in \mathcal{NN}\bigl(2^{m-1}(3k+2), m, 27^m\bigr).
\]

For general $d\geq 2$, we choose $m=\lceil \log_{2}d\rceil$, then $2^{m-1}<d\leq 2^{m}$. We define the target function $\phi$ by
\[
\phi(x):=\phi_{m}\left(
\begin{pmatrix}
\boldsymbol{I}_{d} \\
\boldsymbol{0}_{(2^{m}-d)\times d}
\end{pmatrix}x
+\begin{pmatrix}
\boldsymbol{0}_{d\times 1} \\
\boldsymbol{1}_{(2^{m}-d)\times 1}
\end{pmatrix}\right),
\]
where $\boldsymbol{I}_{d}$ is the $d\times d$ identity matrix, $\boldsymbol{0}_{p\times q}$ is the $p\times q$ zero matrix, and $\boldsymbol{1}_{(2^{m}-d)\times 1}$ is the all-ones vector. In this case, by Lemma  \ref{NN_Com-con},
$\phi\in \mathcal{NN}(2^{m-1}\cdot(3k+2), m, 27^{m})\subseteq  \mathcal{NN}\left(d(3k+2), \lceil \log_{2}d\rceil, 27^{\lceil \log_{2}d\rceil}\right)$.
Furthermore, the approximation error is
\[
\left\|x_{1}\cdots x_{d}-\phi(x)\right\|_{W^{1,\infty}\left([0,1]^{d}\right)}=\left\|x_{1}\cdots x_{2^{m}}-\phi(x)\right\|_{W^{1,\infty}\left([0,1]^{d}\right)}
\leq 4^{m-1}\cdot6 k^{-1}\leq 6d^{2} k^{-1},
\]
where the first step identifies $x_{d+1},\dots,x_{2^m}$ with $1$ in the product, consistent with the definition of $\phi$.
This completes the proof.
\end{proof}

In Lemma \ref{mulproductNNapp}, we construct neural networks to approximate monomials.
For the bump functions $\Psi_m$ as in Definition \ref{PU},
we now construct a neural network $\phi_{m,\alpha}$
to approximate $\Psi_{m}x^{\alpha}$.

\begin{lemma}\label{PUproductNNapp}
Let  $d, N,k\in\NN,$ $\alpha\in \NN_{0}^{d}$ with
$k\geq 4^{\lceil \log_{2}\left(d+\|\alpha\|_{1}\right)\rceil +1}$, and let
$\{\Psi_{m}: m\in\{0,\ldots, N\}^{d}\}$
be the partition of unity
as in Definition \ref{PU}.
Then  there exists a neural network
$
\phi_{m,\alpha}\in\mathcal{NN}((d+\|\alpha\|_{1})(3k+2), \lceil \log_{2}(d+\|\alpha\|_{1})\rceil+1, 60\cdot 27^{\lceil \log_{2}(d+\|\alpha\|_{1})\rceil}N)
$
such that
\[
\left\|\Psi_{m}x^{\alpha}-\phi_{m,\alpha}\right\|_{W^{1,\infty}\left([0,1]^{d}\right)}
\leq CN k^{-1},
\]
for all $m\in\{0,\ldots, N\}^{d}$,
where $C=C(d,\|\alpha\|_1)>0$ is a constant.
\end{lemma}

\begin{proof}
From Definition \ref{PU}, we have the bump functions $\psi$ and $\Psi_m$.
Let $\psi_{m,l}:=\psi\left(3N\left(x_{l}-\frac{m_{l}}{N}\right)\right)$.
Then we have $0\leq\psi_{m,l}\leq 1$ for all $m\in\{0,\ldots, N\}^{d}$ and $x\in \RR^{d}$,
and hence $\Psi_{m}(x)x^{\alpha}=\prod_{l=1}^{d}\psi_{m,l}x^{\alpha}$.
Since $\psi(x)=\sigma(x+2)-\sigma(x+1)-\sigma(x-1)+\sigma(x-2)$ and $x=\sigma(x)-\sigma(-x)$ for any $x\in\RR$,
then $\psi\in\mathcal{NN}(4,1,10)$ and $x\in\mathcal{NN}(2,1,2)$.
By Lemma \ref{NN_Com-con}, we have $\psi_{m,l}\in \mathcal{NN}(4,1,60N)$ for all $m\in\{0,\ldots, N\}^{d}$ and $l=1,\ldots, d$.

Let $D := d+\|\alpha\|_1$ and let $\Phi_{D}\in  \mathcal{NN}(D(3k+2), \lceil \log_{2}D\rceil, 27^{\lceil \log_{2}D\rceil})$
be the $D$-product network constructed in Lemma \ref{mulproductNNapp}.
We define
\[
\phi_{m,\alpha}(x) := \Phi_{d+\|\alpha\|_1}\Bigl(
\psi_{m,1}, \ldots, \psi_{m,d},
\underbrace{x_1, \ldots, x_1}_{\alpha_1 \text{ times}},
\ldots,
\underbrace{x_d, \ldots, x_d}_{\alpha_d \text{ times}}
\Bigr),
\]
where the term $x_{l}$ appears in the input only when $\alpha_{l}\neq 0$ and it repeats  $\alpha_{l}$ times for $l=1,\ldots,d$.
(When $d=1$ and $\alpha=0$, we simply take
$\phi_{m,0}(x)=\psi(3N(x-N^{-1}m))$.)
By Lemma \ref{NN_Com-con},
we have
\[
\phi_{m,\alpha}\in\mathcal{NN}((d+\|\alpha\|_{1})(3k+2), \lceil \log_{2}(d+\|\alpha\|_{1})\rceil+1, 60\cdot 27^{\lceil \log_{2}(d+\|\alpha\|_{1})\rceil}N).
\]
For notational convenience, let
\[
P(y) := \prod_{i=1}^{d+\|\alpha\|_1} y_i, \quad\text{for }
 y=(y_1,\dots,y_{d+\|\alpha\|_1})
\]
and
\[
\psi_{m}^{\alpha}(x) :=\Bigl(
\psi_{m,1}, \ldots, \psi_{m,d},
\underbrace{x_1, \ldots, x_1}_{\alpha_1 \text{ times}},
\ldots,
\underbrace{x_d, \ldots, x_d}_{\alpha_d \text{ times}}
\Bigr).
\]
Applying the chain rule from Lemma \ref{chainandproduct rule} together with Lemma \ref{mulproductNNapp}, the approximation error satisfies
\[
\begin{aligned}
&\left\|\Psi_{m}(x)x^{\alpha}-\phi_{m,\alpha}(x)\right\|_{W^{1,\infty}([0,1]^{d})}\\
&\quad =\left\| P \circ \psi_{m}^{\alpha}(x) - \Phi_{d+\|\alpha\|_{1}} \circ  \psi_{m}^{\alpha}(x) \right\|_{W^{1,\infty}([0,1]^{d})} \\
&\quad= \left\| \left(P -\Phi_{d+\|\alpha\|_{1}}\right) \circ \psi_{m}^{\alpha}(x) \right\|_{W^{1,\infty}([0,1]^{d})} \\
&\quad\leq C \max \left\{ \|P - \Phi_{d+\|\alpha\|_{1}}\|_{L^\infty([0,1]^{d+\|\alpha\|_{1}})},\right.\\
&\qquad\qquad\qquad\left. |P - \Phi_{d+\|\alpha\|_{1}}|_{W^{1,\infty}([0,1]^{d+\|\alpha\|_{1}})} \cdot |\psi_{m}^{\alpha}|_{W^{1,\infty}([0,1]^{d}; [0,1]^{d+\|\alpha\|_{1}})} \right\} \\
&\quad\leq CNk^{-1}.
\end{aligned}
\]
where
the third step uses the fact that $|\psi_{m}^{\alpha}|_{W^{1,\infty}([0,1]^{d}; [0,1]^{d+\|\alpha\|_{1}})}\leq 3N$ since $|\psi_{m,l}|_{W^{1,\infty}([0,1])}
\leq 3N$ for all $m\in\{0,\ldots, N\}^{d}$ and $l=1,\ldots,d$, and
$C=C(d, \|\alpha\|_1)>0$ is a constant.
This completes the proof.
\end{proof}

To construct a local approximation of Sobolev functions using a  partition of unity and polynomials, we present the following lemma, which serves as the foundation for deep  neural network approximation developed in this paper. This lemma follows from \citep[Lemma C.4]{guhring2020error}, where a detailed proof is provided.

\begin{lemma}\label{Le:FuncApprotlocapolyno}
Let $d,N\in\NN, n\in\NN_{\geq 2}$,
$1\leq p\leq\infty$, and let
$\{\Psi_{m}: m\in\{0,\ldots, N\}^{d}\}$
be the partition of unity as in Definition \ref{PU}. Then there is a constant $C=C(d,n,p)>0$ such that the following holds:
For any $f\in W^{n,p}((0,1)^{d})$, there exists a function $f_{m,\alpha}$ defined as
\begin{equation}\label{localpoly}
f_{m,\alpha} :=\sum_{m\in\{0,\ldots, N\}^{d}}\sum_{\|\alpha\|_{1}\leq n-1}c_{f,m,\alpha}\Psi_{m}x^{\alpha}
\end{equation}
such that
$$\left\|f-f_{m,\alpha}\right\|_{W^{1,p}((0,1)^{d})}\leq C\left(\frac{1}{N}\right)^{n-1}\|f\|_{W^{n,p}((0,1)^{d})}.$$
In addition, the polynomial coefficients satisfy
$$|c_{f, m,\alpha}|\leq C \|\widetilde{f}\|_{W^{n,p}(\Omega_{m,N})}N^{d/p},$$
for all $\alpha\in\NN_{0}^{d}$ with $\|\alpha\|_{1}
\leq n-1$ and $m\in\{0,\ldots,N\}^{d}$, where
$\Omega_{m,N}:=B_{\frac{1}{N},\|\cdot\|_{\infty}}\left(\frac{m}{N}\right)$ and $\widetilde{f}\in  W^{n,p}(\RR^{d})$
is an extension of $f$.
\end{lemma}

Now, we can construct a  neural network $\phi$  to approximate $f_{m,\alpha}$.

\begin{lemma}\label{poltNNapp}
Let  $d, N,k\in\NN, n\in\NN_{\geq 2}$, $1\leq p\leq\infty$, with
$k\geq 4^{\lceil \log_{2}\left(d+n-1\right)\rceil +1}$.
Then there is a constant $C=C(d,n,p)>0$ such that the following holds:
For any $f\in \mathcal{F}_{d,n,p}$, let $f_{m,\alpha}$ be defined as in (\ref{localpoly}). Then there exists a neural network
$\phi\in\mathcal{NN}(W,L,K)$ with
\[
\begin{aligned}
W&=C(N+1)^{d}(3k+2),\\
L&=\lceil \log_{2}(d+n-1)\rceil+1,\\
K&=CN^{d/p+1}(N+1)^{d},
\end{aligned}
\]
such that
 \[
 \left\|f_{m,\alpha}-\phi\right\|_{W^{1,p}((0,1)^{d})}
 \leq CN^{d+1}k^{-1}.
\]
\end{lemma}
\begin{proof}
By Lemma \ref{PUproductNNapp}, for all $\alpha\in \NN_{0}^{d}$ with $\|\alpha\|_{1}\leq n-1$ and $m\in\{0,\ldots,N\}^{d}$, $\Psi_{m}x^{\alpha}$
 can be approximated by a neural network $
\phi_{m,\alpha}\in\mathcal{NN}((d+n-1)(3k+2), \lceil \log_{2}(d+n-1)\rceil+1, 60\cdot27^{\lceil \log_{2}(d+n-1)\rceil}N)$.
Now, we can approximate $f_{m,\alpha}$ by
\[
\phi=\sum_{m\in\{0,\ldots, N\}^{d}}\sum_{\|\alpha\|_{1}\leq n-1}c_{f,m,\alpha}
\phi_{m,\alpha}.
\]
The approximation error  is
\begin{align}\label{mixplyapp}
&\left\|f_{m,\alpha}(x)-\phi(x)\right\|_{W^{1,p}((0,1)^{d})}\nonumber\\
&\quad =\left\|\sum_{m\in\{0,\ldots, N\}^{d}}\sum_{\|\alpha\|_{1}\leq n-1}c_{f,m,\alpha}\Psi_{m}(x)x^{\alpha}-
\sum_{m\in\{0,\ldots, N\}^{d}}\sum_{\|\alpha\|_{1}\leq n-1}c_{f,m,\alpha}
\phi_{m,\alpha}(x)\right\|_{W^{1,p}((0,1)^{d})}\nonumber\\
&\quad\leq\sum_{m\in\{0,\ldots, N\}^{d}}\sum_{\|\alpha\|_{1}\leq n-1}|c_{f,m,\alpha}|\left\|\Psi_{m}(x)x^{\alpha}-
\phi_{m,\alpha}(x)\right\|_{W^{1,p}((0,1)^{d})}\nonumber\\
&\quad \leq  \sum_{m\in\{0,\ldots, N\}^{d}}\sum_{\|\alpha\|_{1}\leq n-1}C\|\widetilde{f}\|_{W^{n,p}(\Omega_{m,N})}N^{d/p}
\left\|\Psi_{m}(x)x^{\alpha}-
\phi_{m,\alpha}(x)\right\|_{W^{1,\infty}((0,1)^{d})}\nonumber\\
&\quad\leq CN^{d/p}N k^{-1}\sum_{m\in\{0,\ldots, N\}^{d}}\sum_{\|\alpha\|_{1}\leq n-1}\|\widetilde{f}\|_{W^{n,p}(\Omega_{m,N})}\nonumber\\
&\quad\leq CN^{d/p+1}k^{-1}\sum_{m\in\{0,\ldots, N\}^{d}}\|\widetilde{f}\|_{W^{n,p}(\Omega_{m,N})},
\end{align}
where steps $2$, $3$, and $4$ respectively  use the triangle inequality, Lemma \ref{Le:FuncApprotlocapolyno} together with $\|\cdot\|_{W^{1,p}((0,1)^{d})}
\leq C\|\cdot\|_{W^{1,\infty}((0,1)^{d})}$, and Lemma \ref{PUproductNNapp};
 the last step uses the fact that
$\sum_{\|\alpha\|_{1}\leq n-1}1=\sum_{j=0}^{n-1}\sum_{\|\alpha\|_{1}=j}1
\leq \sum_{j=0}^{n-1}d^{j}\leq nd^{n-1}$. Here
$\widetilde{f}\in  W^{n,p}(\RR^{d})$
is an extension of $f$, $\Omega_{m,N}=B_{\frac{1}{N},\|\cdot\|_{\infty}}\left(\frac{m}{N}\right)$, and $C=C(d,n, p)>0$ is a  constant.
Furthermore, we note that from the definition of $\Omega_{m,N}$ it
follows that there exist $2^{d}$ disjoint subsets ${\mathcal M_{i}}\subset\{0,\ldots, N\}^{d}$
such that $\bigcup_{i=1,\ldots,2^{d}}{\mathcal M_{i}}=\{0,\ldots, N\}^{d}$ and $\Omega_{m_{1},N}\cap\Omega_{m_{2},N}=\varnothing$
for all $m_{1},m_{2}\in{\mathcal M_{i}}$ with $m_{1}\neq m_{2}$ and all $i=1,\ldots,2^{d}$.
From this  we get
\begin{align}\label{multextenfun}
&\sum_{m\in\{0,\ldots, N\}^{d}}\|\widetilde{f}\|_{W^{n,p}(\Omega_{m,N})}\nonumber\\
&\qquad=\sum_{i=1}^{2^{d}}\sum_{m\in{\mathcal M_{i}}} \|\widetilde{f}\|_{W^{n,p}(\Omega_{m,N})}\nonumber\\
&\qquad\le \sum_{i=1}^{2^{d}}\left(\left\lceil \frac{N}{2} \right\rceil + 1\right)^{d(1-1/p)}
\left(\sum_{m\in{\mathcal M_{i}}} \|\widetilde{f}\|_{W^{n,p}(\Omega_{m,N})}^p\right)^{1/p}\nonumber\\
&\qquad \le C\sum_{i=1}^{2^{d}}N^{d(1-1/p)}
\|\widetilde{f}\|_{W^{n,p}\left(\bigcup_{m\in{\mathcal M_{i}}}\Omega_{m,N}\right)}\nonumber\\
&\qquad\le C\sum_{i=1}^{2^{d}}N^{d(1-1/p)}\|\widetilde{f}\|_{W^{n,p}\left(\bigcup_{m\in\{0,\ldots, N\}^{d}}\Omega_{m,N}\right)}\nonumber\\
&\qquad\leq CN^{d(1-1/p)}\|f\|_{W^{n,p}((0,1)^{d})},
\end{align}
where the second step follows from Hölder's inequality and the estimate $|\mathcal{M}_i| \le (\lceil N/2\rceil+1)^d$, the third step uses the additivity of the $p$-th power of the $W^{n,p}$-norm over the disjoint sets $\Omega_{m,N}$ within each $\mathcal{M}_i$, the fourth step uses the inclusion $\bigcup_{m\in\mathcal{M}_i}\Omega_{m,N} \subset \bigcup_{m\in\{0,\ldots,N\}^d}\Omega_{m,N}$, and the last step follows from Remark \ref{Rem: extentheorem}.
Inserting (\ref{multextenfun}) into (\ref{mixplyapp}),  we obtain
\[
 \left\|f_{m,\alpha}(x)-\phi(x)\right\|_{W^{1,p}((0,1)^{d})}
 \leq C\|f\|_{W^{n,p}((0,1)^{d})}N^{d/p+1}N^{d(1-1/p)}k^{-1}
 \leq CN^{d+1}k^{-1},
\]
where $C=C(d, n, p)>0$ is a suitable constant.

By Lemma \ref{Le:FuncApprotlocapolyno} together with Remark \ref{Rem: extentheorem},  we  have
\[
|c_{f, m,\alpha}|\leq C \|\widetilde{f}\|_{W^{n,p}(\Omega_{m,N})}N^{d/p}
\leq C\|f\|_{W^{n,p}((0,1)^{d})}N^{d/p}
\leq CN^{d/p}.
\]
Finally, by Lemma \ref{NN_Com-con}, for all $\alpha\in \NN_{0}^{d}$ with $\|\alpha\|_{1}\leq n-1$ and $m\in\{0,\ldots,N\}^{d}$, we get
$\phi\in\mathcal{NN}(C(N+1)^{d}(3k+2), \lceil \log_{2}(d+n-1)\rceil+1, CN^{d/p+1}(N+1)^{d})$,
where $C=C(d, n, p)>0$.
This completes the proof.

\end{proof}

By combining Lemma \ref{Le:FuncApprotlocapolyno} with Lemma \ref{poltNNapp},
we are now ready to prove Theorem \ref{upperboundapp}.

\begin{proof}[Proof of  Theorem \ref{upperboundapp}]
We divide the proof into two steps.  First, we approximate the function $f$ by a sum of
localized polynomials. Second, we  approximate this sum by a deep norm  constrained neural network.

For the first step, by Lemma \ref{Le:FuncApprotlocapolyno}, we choose
$N=\lceil k^{1/(d+n)}\rceil $. Then,  for any $m\in\{0,\ldots,N\}^{d}$,  there exists a polynomial
$p_{m}(x)=\sum_{\|\alpha\|_{1}\leq n-1}c_{f,m,\alpha}x^{\alpha}$ such that
\begin{align}\label{Appromonomia}
\left\|f-\sum_{m\in\{0,\ldots,N\}^{d}}\Psi_{m}p_{m}\right\|_{W^{1,p}((0,1)^{d})}
&=\left\|f-f_{m,\alpha}\right\|_{W^{1,p}((0,1)^{d})}\nonumber\\
&\leq C\left(\frac{1}{N}\right)^{n-1}
\leq Ck^{-\frac{n-1}{d+n}},
\end{align}
where $C=C(d,n,p)>0$ is a constant.

In the second step, by Lemma \ref{poltNNapp}, for any $k\geq 4^{\lceil \log_{2}\left(d+n-1\right)\rceil +1}$,
there exists a neural network $\phi\in\mathcal{NN}(W,L,K)$ with
\[
\begin{aligned}
W&=C(N+1)^{d}(3k+2)
\asymp k^{(2d+n)/(d+n)},\\
L&=\lceil \log_{2}(d+n-1)\rceil+1,\\
K&
=CN^{d/p+1}(N+1)^{d}
\asymp k^{(d+d/p+1)/(d+n)},
\end{aligned}
\]
 such that
\begin{align}\label{plyapp}
 \left\|\sum_{m\in\{0,\ldots,N\}^{d}}\Psi_{m}p_{m}-\phi\right\|_{W^{1,p}((0,1)^{d})}
 &= \left\|f_{m,\alpha}-\phi\right\|_{W^{1,p}((0,1)^{d})}\nonumber\\
 &\leq CN^{d+1}k^{-1}
 \leq  Ck^{-\frac{n-1}{d+n}},
\end{align}
where $C=C(d,n,p)>0$ is a constant.
Therefore, we have
\[
k\asymp K^{(d+n)/(d+d/p+1)},
\]
and consequently,
\[
W\asymp k^{(2d+n)/(d+n)}
\asymp K^{(2d+n)/(d+d/p+1)}.
\]
When $K\gtrsim 4^{(\lceil \log_{2}(d+n-1)\rceil +1)(d+d/p+1)/(d+n)}$, applying the triangle inequality  together with (\ref{Appromonomia}) and (\ref{plyapp}), we get
\[
\begin{aligned}
&\left\|f-\phi\right\|_{W^{1,p}((0,1)^{d})}\\
&\quad\leq\left\|f-\sum_{m\in\{0,\ldots,N\}^{d}}\Psi_{m}p_{m}\right\|_{W^{1,p}((0,1)^{d})}
+\left\|\sum_{m\in\{0,\ldots,N\}^{d}}\Psi_{m}p_{m}-\phi\right\|_{W^{1,p}((0,1)^{d})}\\
&\quad  \lesssim k^{-\frac{n-1}{d+n}}\\
&\quad\lesssim  K^{-\frac{n-1}{d+d/p+1}}.
\end{aligned}
\]
Since increasing $W$ and $L$ can only reduce the approximation error, the bound holds for any
$W\gtrsim K^{(2d+n)/(d+d/p+1)}$ and $L\geq\lceil \log_{2}(d+n-1)\rceil+1$.
This completes the proof.
\end{proof}

\section{Conclusions}\label{Conclusions}

This paper has established approximation error bounds for norm constrained neural networks in the $W^{1,p}$-norm. For shallow norm constrained networks, we showed that the error bound decays polynomially in both the width and the path norm constraint, provided the parameters satisfy $n<s=(d+3)/2$ and $2\le p\le\infty$. For deep norm constrained networks, we proved that the error bound decays polynomially in the path norm constraint, independent of the network width and depth, making the result particularly suitable for over-parameterized settings.

Several open questions remain. The present analysis is restricted to the regime $n<s=(d+3)/2$ and $2\le p\le\infty$ for shallow norm constrained  networks; extending the results to $n \ge (d+3)/2$ or $1\le p<2$ would be an interesting direction for future work. Moreover,
whether deep norm constrained networks can achieve even better approximation rates remains an open problem.
Finally, it would be worthwhile to investigate whether the norm constrained approximation framework can be further extended to other function classes.

\section*{Acknowledgments}
The work described in this paper was partially supported by National Natural Science Foundation of China under Grants 12501131 and 12526216.
\begin{appendix}
\section{Some useful definitions and lemmas}

In this appendix, we provide some definitions and lemmas, which will be
used in the proof of our main results.

We give the definition of the normalized smooth radially symmetric bump function, which will be used in Section \ref{Proofofresult1}.
\begin{definition}\label{Def:smooradsymbu}
A function $\phi : \mathbb{R}^d \to [0,\infty)$ is called a \emph{normalized smooth radially symmetric bump function supported on}  $\overline{\BB^d}$,  if it satisfies the following properties:
\begin{enumerate}[label=\textnormal{(\arabic*)}]
\item $\phi \in C_c^\infty(\mathbb{R}^d)$
and $\phi(x) \ge 0$ for all $x \in \mathbb{R}^d$;
\item $\operatorname{supp} \phi = \overline{\BB^d}$;
\item $\int_{\mathbb{R}^d} \phi(x)dx = 1$;
\item $\phi$ is radial, i.e., $\phi(x) = \phi(y)$ whenever $\|x\|_{2} = \|y\|_{2}$.
\end{enumerate}
\end{definition}
A normalized smooth radially symmetric bump function supported on  $\overline{\BB^d}$ always exists. A possible choice is
\[
\psi(x) = \begin{cases}
\exp\left( -\frac{1}{1 - \|x\|_2^2} \right), & \|x\|_2 < 1, \\
0, & \text{otherwise},
\end{cases}
\]
and then define $\phi(x) = \bigl( \int_{\mathbb{R}^d} \psi(y)dy \bigr)^{-1} \psi(x)$, so that $\operatorname{supp} \phi = \overline{\BB^d}$ and $\int_{\mathbb{R}^d} \phi(x)dx = 1$.

We introduce the notion of partition  of unity (in the same way as in \citep[Theorem 1]{yarotsky2017error} and \citep[Lemma C.3]{guhring2020error}) which can be defined as a tensor product of piecewise linear functions. This concept plays a crucial role in localizing function approximations.
\begin{definition}\label{PU}
Define the one-dimensional  bump function $\psi: \mathbb{R} \to \mathbb{R}$ by
\[
\psi(x) := \begin{cases}
1, & |x| < 1,\\
0, & 2 < |x|,\\
2 - |x|, & 1 \le |x| \le 2.
\end{cases}
\]
For $d, N \in \mathbb{N}$ and $m=(m_{1},\ldots,m_{d})\in \{0,1,\dots,N\}^d$, define the $d$-dimensional  bump function $\Psi_{m}: \mathbb{R}^d \to \mathbb{R}$ as a tensor product of scaled and shifted versions of $\psi$:
\begin{equation}\label{PU1}
\Psi_{m}(x) := \prod_{l=1}^d \psi\!\left(3N\Bigl(x_l - \frac{m_l}{N}\Bigr)\right), \qquad x = (x_1,\dots,x_d)^\top \in \mathbb{R}^d.
\end{equation}
The family $\{\Psi_{m} : m \in \{0,\dots,N\}^d\}$
forms a partition of unity on $[0,1]^d$ with the following properties:
\begin{enumerate}[label=\textnormal{(\arabic*)}]
    \item $0 \le \Psi_{m}(x) \le 1$ for all $m$ and all $x \in \mathbb{R}^d$;
    \item $\displaystyle\sum_{m \in \{0,\dots,N\}^d} \Psi_{m}(x) = 1$ for all $x \in [0,1]^d$;
    \item $\|\Psi_{m}\|_{W^{1,\infty}(\RR^{d})}\le 3N$ for all $m$;
    \item $\operatorname{supp} \Psi_{m} \subset B_{\frac1N,\|\cdot\|_\infty}\left(\frac{m}{N}\right)$ for all $m$.
\end{enumerate}
\end{definition}

The following lemma, which
follows from \citep[Corollary B.5]{guhring2020error},
establishes a chain rule  estimate for composite  functions in $W^{1,\infty}$.

\begin{lemma}\label{chainandproduct rule}
Let $d_1,d_2\in \mathbb{N}$, and let $\Omega_1\subset \mathbb{R}^{d_1}$, $\Omega_2\subset \mathbb{R}^{d_2}$ be open, bounded, and convex domains.
Then there exists  a constant $C=C(d_1,d_2)>0$  such that the following  estimate holds:
Let $f\in W^{1,\infty}(\Omega_1;\mathbb{R}^{d_2})$ and $g\in W^{1,\infty}(\Omega_2)$ be Lipschitz functions
with $\mathrm{range}(f)\subset \Omega_2$. Then $g\circ f\in W^{1,\infty}(\Omega_1)$, and
\begin{equation}\label{chainrule}
\|g\circ f\|_{W^{1,\infty}(\Omega_1)}\leq C\max\left\{\|g\|_{L^{\infty}(\Omega_2)},|g|_{W^{1,\infty}(\Omega_2)}\cdot |f|_{W^{1,\infty}(\Omega_1;\mathbb{R}^{d_2})}\right\}.
\end{equation}
\end{lemma}

The next lemma, which follows from \citep[Proposition 2.5]{jiao2023approximation}, summarizes some basic operations on neural networks.  These operations will be useful for constructing neural networks in the study of approximation capacity in Section \ref{ApproximadeepReLU}.

\begin{lemma}\label{NN_Com-con}
Let $\phi_{1}\in  \mathcal{NN}_{d_{1},m_{1}}(W_{1},L_{1},K_{1})$ and $\phi_{2}\in  \mathcal{NN}_{d_{2},m_{2}}(W_{2},L_{2},K_{2})$.
\begin{enumerate}[label=\textnormal{(\arabic*)}]
\item If $d_{1}=d_{2}$, $m_{1}=m_{2}$, $W_{1}\leq W_{2}$, $L_{1}\leq L_{2}$ and $K_{1}\leq K_{2}$, then
$\mathcal{NN}_{d_{1},m_{1}}(W_{1},L_{1},K_{1})\subseteq \mathcal{NN}_{d_{2},m_{2}}(W_{2},L_{2},K_{2})$.
\item \textnormal{\textbf{(Composition)}} If $m_{1}=d_{2}$, then $\phi_{2}  \circ \phi_{1}
\in\mathcal{NN}_{d_{1},m_{2}}(\max\{W_{1},W_{2}\},L_{1}+L_{2}, K_{2}\max \{K_{1},1\})$.  Let $A\in \RR^{d_{2}\times d_{1}}$ and $b\in\RR^{d_{2}}$. Define the function
$\phi({x}):=\phi_{2}(A x+b)$ for $x \in \RR^{d_{1}}$, then $\phi \in \mathcal{NN}_{d_{1},m_{2}}(W_{2},L_{2},
K_{2}\max \{\|(A,b)\|,1\})$.
\item \textnormal{\textbf{(Concatenation)}}
If $d_{1}=d_{2}$, define  $\phi(x):=( \phi_{1}(x),  \phi_{2}(x))$,
then $\phi\in\mathcal{NN}_{d_{1},m_{1}+m_{2}}(W_{1}+W_{2},\max\{L_{1},L_{2}\}, \max \{K_{1},K_{2}\})$.
\item \textnormal{\textbf{(Linear Combination)}}
If $d_{1}=d_{2}$ and $m_{1}=m_{2}$, then, for any $c_{1},c_{2}\in\RR$,
 $c_{1}\phi_{1}+c_{2}\phi_{2}\in\mathcal{NN}_{d_{1},m_{1}}(W_{1}+W_{2},\max\{L_{1},L_{2}\}, |c_{1}|K_{1}+|c_{2}|K_{2})$.
\end{enumerate}
\end{lemma}

\end{appendix}

\bibliographystyle{myplainnat}
\bibliography{references}
\end{document}